\documentclass[11pt]{article}

\usepackage{amsmath,amsbsy,amsfonts,amssymb,amsthm,dsfont,fullpage,units}
\usepackage{xspace}
\usepackage[authoryear]{natbib}
\usepackage{graphicx}
\usepackage{algorithm,algorithmic,mathtools}
\usepackage{color,cases}
\usepackage{tikz}
\usetikzlibrary{calc,shapes}
\usepackage{subfigure}

\usepackage{hyperref}
\hypersetup{
colorlinks = true,
citecolor = {blue},
citebordercolor = {white}
}

\newcommand{\MG}{\textsc{MIX+GAN}\xspace}
\newcommand{\DCGAN}{\textsc{DCGAN}\xspace}
\newcommand{\MDC}{\textsc{MIX+DCGAN}\xspace}
\newcommand{\WGAN}{\textsc{WassersteinGAN}\xspace}
\newcommand{\MW}{\textsc{MIX+WassersteinGAN}\xspace}

\def\shownotes{0}  
\ifnum\shownotes=1
\newcommand{\authnote}[2]{$\ll$\textsf{\footnotesize #1 notes: #2}$\gg$}
\else
\newcommand{\authnote}[2]{}
\fi

\newtheorem{claim}{Claim}
\newtheorem{corollary}{Corollary}[section]

\newtheorem{lemma}{Lemma}

\newtheorem{theorem}{Theorem}[section]
\theoremstyle{definition}

\newtheorem{definition}{Definition}
\newtheorem{example}{Example}

 \newcommand{\IGNORE}[1]{}

\DeclareMathOperator*\E{\mathbb{E}}
\newcommand\R{\mathbb{R}}
\renewcommand\S{\mathcal{S}}
\newcommand\N{\mathcal{N}}
\newcommand\D{\mathcal{D}}
\newcommand\U{\mathcal{U}}
\newcommand\V{\mathcal{V}}
\newcommand\X{\mathcal{X}}

\newcommand\eps{\varepsilon}

\newcommand{\Dreal}{\cD_{real}}
\newcommand{\hatDreal}{\hat{\cD}_{real}}
\newcommand{\cDr}{\mathcal D_{real}}
\newcommand{\cD}{\mathcal D}

\title{Generalization and  Equilibrium in Generative Adversarial Nets (GANs)}
\date{}
\author{Sanjeev Arora\thanks{Princeton University, Computer Science Department, email: arora@cs.princeton.edu} \and Rong Ge \thanks{Duke University, Computer Science Department, email: rongge@cs.duke.edu}\and Yingyu Liang\thanks{Princeton University, Computer Science Department, email: yingyul@cs.princeton.edu}\and Tengyu Ma\thanks{Princeton University, Computer Science Department, email: tengyu@cs.princeton.edu} \and Yi Zhang\thanks{Princeton University, Computer Science Department, email: yz7@cs.princeton.edu}}

\begin{document}
\maketitle
\begin{abstract} 

We show that  training of generative adversarial network (GAN) may not have good {\em generalization} properties;
e.g., training may appear successful but  the trained distribution may be  far from  target distribution in standard metrics. However, generalization does occur for a weaker metric called {\em neural net distance}. It is also shown that an approximate pure equilibrium exists\footnote{This is an updated version of an ICML'17 paper with the same title. The main difference is that in the ICML'17 version the pure equilibrium result was only proved for Wasserstein GAN. In the current version the result applies to most reasonable training objectives. In particular, Theorem~\ref{thm:pureequilibrium} now applies to both original GAN and Wasserstein GAN.} in the discriminator/generator game for a special class of generators with natural training objectives when generator capacity and training set sizes are moderate.

This existence of equilibrium inspires {\sc mix+gan} protocol, which can be combined with any existing GAN training, and empirically shown to improve some of them.
\end{abstract}
\section{Introduction}

Generative Adversarial Networks (GANs)~\citep{goodfellow2014generative} have become one of the dominant methods for fitting generative models to complicated real-life data, and even found unusual uses such as designing good cryptographic primitives~\citep{abadi2016learning}. See a survey by \cite{goodfellow2016nips}.
Various novel architectures and training objectives were introduced to address perceived shortcomings of the original idea, leading to more stable training and more realistic generative models in practice (see ~\cite{odena2016conditional, xh2016SGAN,soumith2016DCGAN,tolstikhin2017adagan,Salimans2016ImprovedGANs,2016arXiv161204021J,2016arXiv161101673D} and the reference therein). 

The goal is to train a {\em generator} deep net whose input is a standard Gaussian, and whose output is a sample from some distribution $\cD$ on $\R^d$, which has to be close to some target distribution $\cDr$ (which could be, say, real-life images represented using raw pixels). The training uses samples from $\cDr$ and 
together with the generator net also trains a {\em discriminator} deep net trying to maximise its ability to distinguish between samples from $\cDr$ and $\cD$. So long as the discriminator is successful at this task with nonzero probability, its success can be used to generate a feedback (using backpropagation) to the generator, thus improving its distribution $\cD$. Training is continued until the generator {\em wins}, meaning that the discriminator 
can do no better than random guessing when deciding whether or not a particular sample came from $\cD$ or $\cDr$.
This basic iterative framework has been tried with many training objectives; see Section~\ref{sec:prelim}. 
\begin{figure}[!t]
	\centering
	\includegraphics[height=1.3in, width=2in]{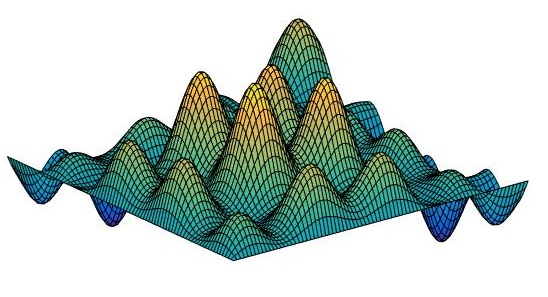}
	\caption{Probability density $\cDr$ with many peaks and valleys}
	\label{fig:peaks}
\end{figure}
But it has been unclear what to conclude when the generator wins this game: is $\cD$ close to $\cDr$ in some metric? 
One seems to need some extension of {\em generalization theory} that would imply such a  conclusion.
The hurdle  is that distribution $\cDr$ could be complicated and may have many peaks and valleys; see Figure~\ref{fig:peaks}. The number of peaks (modes) may even be exponential in $d$. (Recall the {\em curse of dimensionality}: in $d$ dimensions there are $\exp(d)$ directions whose pairwise angle exceeds say $\pi/3$, and each could be the site of a peak.) Whereas the number of samples from $\cDr$ (and from $\cD$ for that matter) used in the  training  is a lot fewer, and thus may not reflect most of the peaks and valleys of $\cDr$. 

A standard analysis due to~\citep{goodfellow2014generative} shows that when the discriminator capacity (= number of parameters) and number of samples is \textquotedblleft large enough\textquotedblright, then a win by the generator  implies that $\cD$ is very close to $\cDr$  (see Section~\ref{sec:prelim}). But the discussion in the previous paragraph raises the possibility that \textquotedblleft sufficiently large\textquotedblright\ in this analysis may need to be $\exp(d)$.

Another open theoretical issue  is whether an equilibrium  always exists in this game between generator and discriminator. Just as a zero gradient is a necessary condition for standard optimization to halt, the corresponding necessary condition in a two-player game is an equilibrium.  Conceivably some of the instability often observed while training GANs could just arise because of lack of equilibrium.  (Recently \citet{arjovsky2017wasserstein} suggest that using their  Wasserstein objective in practice reduces instability, but we still lack proof of existence of an equilibrium.)
Standard game theory is of no help here because we need a so-called pure equilibrium, and  simple counter-examples such as  {\em rock/paper/scissors}  show that it doesn't exist in general\footnote{Such counterexamples are easily turned into toy GAN scenarios with generator and discriminator having finite capacity, and the game lacks a pure equilibrium. See Appendix~\ref{sec:examples}.}.

\subsection{Our Contributions}

We formally define generalization for GANs in Section~\ref{sec:generalize} and show that for previously studied notions of distance between distributions, generalization is not guaranteed (Lemma~\ref{lem:warmup}). In fact we show that the generator can win even when $\cD$ and $\cDr$ are arbitrarily far in any one of the standard metrics. 

However, we can guarantee some weaker notion of generalization by introducing a  new metric on distributions, the {\em neural net distance}. We show that generalization does happen with moderate number of training examples (i.e., when the generator wins, the two distributions must be close in neural net distance). However, this weaker metric comes at a cost: it can be near-zero even when the trained and target distributions are very far (Section~\ref{sec:diversity}). 

To explore the existence of equilibria we turn in Section~\ref{sec:infinitemix} to infinite mixtures of generator deep nets. These are clearly vastly more expressive than a single generator net: e.g., a standard result in bayesian nonparametrics says that every probability density is closely approximable by an infinite mixture of Gaussians~\citep{ghosh2003bayesian}. Thus unsurprisingly, an infinite mixture should win the game. We then prove rigorously that even a finite mixture of fairly reasonable size can closely approximate the performance of  the infinite mixture (Theorem~\ref{thm:mixedequilibrium}). 

This insight also allows us to construct a new architecture for the generator network where % show for a natural GAN setting with Wasserstein objective 
there exists an {\em approximate} equilibrium that is pure. (Roughly speaking, an approximate equilibrium is one in which neither of the players can gain much by deviating from their strategies.) 
This existence proof for an approximate equilibrium unfortunately involves a quadratic blowup in the \textquotedblleft size\textquotedblright\ of the generator (which is still better than the naive exponential blowup one might expect). Improving this is left for future theoretical work. But we propose a heuristic approximation to the mixture idea to introduce a new framework for training that we call \MG. It can be added on top of any existing GAN training procedure, including those that use divergence objectives. Experiments in Section~\ref{sec:experiments} show that for several previous techniques, \MG stabilizes  the training, and in some cases improves the performance. 
\label{sec:intro}
\section{Preliminaries}
\label{sec:prelim}

\noindent{\bf Notations.} Throughout the paper we use $d$ for the dimension of samples, and $p$ for the number of parameters in the generator/discriminator. In Section~\ref{sec:generalize} we use $m$ for number of samples. \\

\noindent{\bf Generators and discriminators.} 
Let $\{G_u, u\in\mathcal{U}\}$ ($\mathcal{U}\subset \R^p$) denote the class of generators, where $G_u$ is a function --- which is often a neural network in practice --- from $\R^\ell \rightarrow \R^d$ indexed by $u$ that denotes the parameters of the generators. Here $\mathcal{U}$ denotes the possible ranges of the parameters and without loss of generality we assume $\mathcal{U}$ is a subset of the unit ball\footnote{Otherwise we can scale the parameter properly by changing the parameterization. }. The generator $G_u$ defines a distribution $\cD_{G_u}$ as follows: generate $h$ from $\ell$-dimensional spherical Gaussian distribution and then apply $G_u$ on $h$ and generate a sample $x = G_u(h)$ of the distribution $\cD_{G_u}$. We drop the subscript $u$ in $\cD_{G_u}$ when it's clear from context.

Let $\{D_v, v\in\mathcal{V}\}$ denote the class of discriminators, where $D_v$ is function from $\R^d$ to $[0,1]$ and $v$ is the parameters of $D_v$. 

Training the discriminator consists of trying to make it output a high value (preferably $1$) when $x$ is sampled from distribution $\mathcal{D}_{real}$  and a low value (preferably $0$) when $x$ is sampled from the synthetic distribution $\cD_{G_u}$. Training the discriminator consists of trying to make its synthetic distribution  \textquotedblleft similar\textquotedblright to $\mathcal{D}_{real}$  in the sense that the discriminator's output tends to be similar on the two distributions.

We assume $G_u$ and $D_v$ are $L$-Lipschitz with respect to their parameters. That is, for all $u,u'\in \mathcal{U}$ and any input $h$, we have $\|G_u(h)-G_{u'}(h)\| \le L\|u-u'\|$ (similar for $D$).

Notice, this is distinct from the assumption (which we will also sometimes make) that functions $G_u, D_v$ are Lipschitz: that focuses on the change in function value when we change $x$, while keeping $u,v$ fixed\footnote{Both Lipschitz parameters can be exponential in the number of layers in the neural net, however our Theorems only depend on the $\log$ of the Lipschitz parameters}.

\noindent{\bf Objective functions.}
The standard GAN training~\citep{goodfellow2014generative} consists of training parameters $u, v$ so as to optimize an objective function:
\begin{equation}\label{eq:gan}
\min_{u\in \mathcal{U}} \max_{v\in\mathcal{V}} \E_{x\sim \mathcal{D}_{real}}[\log D_v(x)] + \E_{x\sim \cD_{G_u}}[\log (1-D_v(x))].
\end{equation}
Intuitively, this says that the discriminator $D_v$ should give high values $D_v(x)$ to the real samples and low values $D_v(x)$ to the generated examples.  The $\log$ function was suggested because of its interpretation as the likelihood, and it also has a nice information-theoretic interpretation described below. However, 
in practice it can cause problems since $\log x \rightarrow -\infty$ as $x \rightarrow 0$.  The objective still makes intuitive sense if we replace $\log$ by any monotone function $\phi:[0,1]\to \R$, which yields the objective: \begin{equation} \label{eq:gan-general}
\min_{u\in \mathcal{U}} \max_{v\in\mathcal{V}} \E_{x\sim \mathcal{D}_{real}}[\phi(D_v(x))] + \E_{x\sim \cD_{G_u}}[\phi(1-D_v(x))].
\end{equation}
We call function $\phi$ the {\em measuring function}. It should be concave so that when $\D_{real}$ and $\cD_G$ are the same distribution, the best strategy for the discriminator is just to output $1/2$ and the optimal value is $2\phi(1/2)$. In later proofs, we will require $\phi$ to be bounded and Lipschitz. Indeed, in practice training often uses  $\phi(x) = \log (\delta + (1-\delta)x)$ (which takes values in $[\log \delta,0]$ and is $1/\delta$-Lipschitz) and the recently proposed Wasserstein GAN~\citep{arjovsky2017wasserstein} objective uses $\phi(x) = x$. 

\noindent{\bf Training with finite samples.}
The objective function~\eqref{eq:gan-general} assumes we have infinite number of samples from $\cD_{real}$ to estimate the value $\E_{x\sim \mathcal{D}_{real}}[\phi(D_v(x))]$. With finite training examples $x_1,\dots, x_m \sim \cD_{real}$,  one uses $\frac{1}{m}\sum_{i=1}^m [\phi(D_v(x_i))]$ to estimate the quantity $\E_{x\sim \cD_{real}}[\phi(D_v(x))]$. We call the distribution that gives $1/m$ probability to each of the $x_i$'s the {\em empirical version} of the real distribution. Similarly, one can use a empirical version to estimate $\E_{x\sim \cD_{G_u}}[\phi (1-D_v(x))].$

\noindent{\bf Standard interpretation via distance between distributions.} Towards analyzing GANs, researchers have assumed access to \textit{infinite} number of examples and that the discriminator is chosen optimally within some \textit{large} class of functions that contain all possible neural nets. This often allows computing analytically the optimal discriminator and therefore removing the maximum operation from the objective~\eqref{eq:gan-general}, which leads to some interpretation of how and in what sense the resulting distribution $\cD_G$ is close to the true distribution $\mathcal{D}_{real}$. 

Using the original objective function~\eqref{eq:gan}, then the optimal choice among all the possible functions from $\R^d\rightarrow (0,1)$ is $D(x) = \frac{P_{real}(x)}{P_{real}(x)+P_G(x)}$, as shown in~\cite{goodfellow2014generative}. Here $P_{real}(x)$ is the density of $x$ in the real distribution, and $P_G(x)$ is the density of $x$ in the distribution generated by generator $G$. Using this discriminator --- though it's computationally infeasible to obtain it --- one can show that the minimization problem over the generator correspond to minimizing the Jensen-Shannon (JS) divergence between the true distribution $\mathcal{D}_{real}$ and the generative distribution $\cD_G$. Recall that for two distributions $\mu$ and $\nu$, the JS divergence is defined by
$$
d_{JS}(\mu,\nu) = \frac{1}{2}(KL(\mu\|\frac{\mu+\nu}{2})+KL(\nu\|\frac{\mu+\nu}{2}))\,.
$$

Other measuring functions $\phi$ and choice of discriminator class leads to different distance function between distribution other than JS divergence. Notably, ~\cite{arjovsky2017wasserstein} shows that when $\phi(t) = t$, and the discriminator is chosen among all $1$-Lipschitz functions, maxing out the discriminator, the generator is attempting to minimize the Wasserstein distance between $\cD_{real}$ and $D_u(h)$. Recall that Wasserstein distance between $\mu$ and $\nu$ is defined as $$d_{W}(\mu,\nu) = \sup_{D\mbox{~is 1-Lipschitz}} \left|\E_{x\sim \mu}[D(x)] - \E_{x\sim \nu}[D(x)]\right|\,.$$

\section{Generalization theory for GANs}
\label{sec:generalize}

The above interpretation of GANs in terms of  minimizing distance (such as JS divergence and Wasserstein distance) between the real distribution and the generated distribution relies on two crucial assumptions: (i)  very expressive class of discriminators such as the set of all bounded discriminator or the set of all 1-Lipschitz discriminators, and (ii) very large number of examples to compute/estimate the objective~\eqref{eq:gan} or \eqref{eq:gan-general}.
Neither assumption holds in practice, and we will show next that this greatly affects the generalization ability, a notion we introduce in Section~\ref{sec:def-gen}.

%One of the messages of our analysis will be that disciminators with bounded capacity perform very differently than the ideal (strong) discriminators analysed in Section~\ref{sec:prelim}. First, in Section~\ref{sec:fail-to-generalize} we show that the JS divergence and Wasserstein distance that resulted from strong discriminators (such as bounded or 1-Lipschitz functions) don't generalize, whereas in Section~\ref{sec:neural-net-gen} we show that discriminator class with bounded number of parameters (such as neural nets) do generalize with relatively modest number of examples. 

%Moreover, even though discriminator with capacity $p$ has good generalization power, there might be a potential diversity issue with weak discriminators. We will see (Corollary~\ref{corr:lackofdiversity} in Section~\ref{sec:diversity}) that a discriminator of capacity $p$ cannot distinguish too well between $\D_{real}$ and the empirical distribution $\hat{\D}_{real}$ when the number of samples $m$ exceeds $(p\log p)/\epsilon^2$. 
%This holds even we have access to infinite number of examples from $\D_{real}$. 
%So this is not a question of \textquotedblleft insufficient training samples,\textquotedblright nor the usual worry of 
%\textquotedblleft overfitting.\textquotedblright One way to view this result is that
%a bounded capacity discriminator is unable to force the generator to produce a distribution with very high diversity. 
%We note that similar results have been shown before in study of pseudorandomness~\citep{trevisan2009regularity} and model criticism~\citep{gretton2012kernel}.
\subsection{Definition of Generalization}\label{sec:def-gen}

Our definition is motivated from supervised classification, where training is said to generalize if the training and test error closely track each other. (Since the purpose of GANs training is to learn a distribution, one could also consider a stronger definition of successful training, as discussed in Section~\ref{sec:diversity}.)

Let  $x_1,\dots, x_m$ be the training examples, and let $\hat{\cD}_{real}$ denote the uniform distribution over $x_1,\dots, x_m$. Similarly, let $G_u(h_1),\dots, G_u(h_{r})$ be a set of $r$ examples from the generated distribution $\cD_G$.  In the training of GANs, one implicitly uses $\E_{x\sim \hatDreal} [\phi(D_v(x))]$ to approximate the quantity $\E_{x\sim \Dreal} [\phi(D_v(x))]$. Inspired by the observation that the training objective of GANs and its variants is to minimize some distance (or divergence) $d(\cdot,\cdot)$ between $\Dreal$ and $\cD_G$ using finite samples, we define the generalization of GANs as follows: 

\begin{definition}\label{def:generalization}
	Given $\hatDreal$, an empirical version of the true distribution with $m$ samples, a generated distribution $\cD_G$ {\em generalizes} under the divergence or distance between distributions $d(\cdot,\cdot)$  with generalization error $\eps$ if the following holds with high probability\footnote{over the choice of $\hat{\cD}_G$}, 
	\begin{equation}
	\left|d(\Dreal,\cD_G) - d(\hatDreal,\hat{\cD}_G) \right| \le  \eps \label{eqn:generalization}
	\end{equation}
	where and $\hat{\cD}_G$ is an empirical version of the generated distribution $\cD_G$ with polynomial number of samples (drawn after $\cD_G$ is fixed). 
\end{definition}
In words, generalization in GANs means that the population distance between the true and generated distribution is close to the empirical distance between the empirical distributions. Our target is to make the former distance small, whereas the latter one is what we can access and minimize in practice. The definition allows only polynomial number of samples from the generated distribution because the training algorithm should run in polynomial time.  

We also note that stronger versions of Definition~\ref{def:generalization} can be considered. For example, as an analog of uniform convergence in supervised learning, we can require~\eqref{eqn:generalization} to hold for \textit{all} generators $\cD_G$ among a class of candidate generators. Indeed, our results in Section~\ref{sec:neural-net-gen} show that \textit{all} generators generalize under neural net distance with reasonable number of examples. 

\subsection{JS Divergence and Wasserstein don't Generalize}\label{sec:fail-to-generalize}

As a warm-up, we show that JS divergence and Wasserstein distance don't generalize with any polynomial number of examples because the population distance (divergence) is not reflected by the empirical distance. \begin{lemma}\label{lem:warmup}
	Let $\mu$ be uniform Gaussian distributions $\N(0,\frac{1}{d}I)$ and $\hat{\mu}$ be an empirical versions of $\mu$ with $m$ examples. Then we have $d_{JS}(\mu,\hat{\mu}) = \log 2$, $d_{W}(\mu,\hat{\mu}) \ge 1.1$.
\end{lemma}
There are two consequences of Lemma 1. First, consider the situation where $\Dreal = \cD_G =  \mu$. Then we have that $d_W(\Dreal, \cD_G) = 0$ but $d_W(\hatDreal,\hat{\cD}_G) > 1$ as long as we have polynomial number of examples. This violates the generalization definition equation ~\eqref{eqn:generalization}. 

Second, consider the case $\Dreal = \mu$ and $\cD_G = \hatDreal = \hat{\mu}$, that is, $\cD_G$ memorizes all of the training examples in $\hatDreal$. In this case, since $\cD_G$ is a discrete distribution with finite supports, with enough (polynomial) examples, in $\hat{\cD}_G$, effectively we also have that $\hat{\cD}_G \approx \cD_G$. Therefore, we have that $d_W(\hatDreal,\hat{\cD}_G) \approx 0$ whereas $d_W(\Dreal, \cD_G) > 1$. In other words, with any polynomial number of examples, it's possible to overfit to the training examples using Wasserstein distance. The same argument also applies to JS divergence. See Appendix~\ref{sec:app:generalize} for the formal proof.
Notice, this result does not contradict the experiments of~\cite{arjovsky2017wasserstein} since they actually use not Wasserstein distance but  a surrogate distance that does generalize, as we show next.

\subsection{Generalization bounds for neural net distance}\label{sec:neural-net-gen}

Which distance measure between $\Dreal$ and $\cD_G$ is the GAN objective actually minimizing and can we analyze its generalization performance? Towards answering these questions in full generality (given multiple GANs objectives) we consider the following general distance measure that unifies JS divergence, Wasserstein distance, and the neural net distance that we define later in this section.

\begin{definition}[$\mathcal{F}$-distance]
	Let $\mathcal{F}$ be a class of functions from $\R^d$ to $[0,1]$ such that if $f\in \mathcal{F}, 1-f\in \mathcal{F}$. Let $\phi$ be a concave measuring function. Then the $\mathcal{F}${\em -divergence with respect to $\phi$ between two distributions} $\mu$ and $\nu$ supported on $\R^d$ is defined as 
	\begin{align}
	d_{\mathcal{F},\phi}(\mu, \nu) = & \sup_{D\in \mathcal{F}} \E_{x\sim \mu }[\phi(D(x))] + \E_{x\sim \nu}[\phi(1-D(x))] - 2\phi(1/2) \nonumber
	\end{align}
	When $\phi(t) = t$, we have that $d_{\mathcal{F},\phi}$ is a distance function \footnote{Technically it is a pseudometric. This is also known as integral probability metrics\citep{muller1997integral}.}  	, and with slightly abuse of notation we write it simply as $d_{\mathcal{F}}(\mu,\nu) = \sup_{D\in \mathcal{F}} \left|\E_{x\sim \mu }[D(x)] - \E_{x\sim \nu}[D(x)]\right|.$
\end{definition}

\begin{example}
	When $\phi(t) = \log (t)$ and $\mathcal{F} = \{\textrm{all functions from $\R^d$ to $[0,1]$}\}$, we have that $d_{\mathcal{F}, \phi}$ is the same as JS divergence. When $\phi(t) = t$ and $\mathcal{F} = \{\textrm{all 1-Lipschitz functions from $\R^d$ to $[0,1]$}\}$, then $d_{\mathcal{F},\phi}$ is the Wasserstein distance. 
\end{example}

\begin{example}
	Suppose $\mathcal{F}$ is a set of neural networks and $\phi(t) = \log t$, then original GAN objective function is equivalent to $\min_{G} d_{\mathcal{F}, \phi}(\hatDreal, \hat{\cD}_G)\,.$
	
	Suppose $\mathcal{F}$ is the set of neural networks, and $\phi(t) = t$, then the objective function used empirically in~\cite{arjovsky2017wasserstein} is equivalent to $\min_{G} d_{\mathcal{F}}(\hatDreal, \hat{\cD}_G)\,.$
	
\end{example}

GANs training uses $\mathcal{F}$ to be a class of neural nets with a bound $p$ on the number of parameters. We then informally refer to $d_{\mathcal{F}}$ as the neural net distance. 
The next theorem establishes generalization in the sense of equation~\eqref{eqn:generalization} does hold for it (with a uniform convergence) .
We assume that the measuring function takes values in $[-\Delta,\Delta]$ and that it is $L_\phi$-Lipschitz. Further, $\mathcal{F} = \{D_v, v\in \mathcal{V}\}$ is the class of discriminators that is $L$-Lipschitz with respect to the parameters $v$. As usual, we use $p$ to denote the number of parameters in $v$. 

\begin{theorem}
	\label{thm:nngeneralizesingle}
	In the setting of previous paragraph, let $\mu,\nu$ be two distributions and $\hat{\mu},\hat{\nu}$ be empirical versions with at least $m$ samples each. There is a universal constant $c$ such that  when $m \ge \frac{cp\Delta^2\log (LL_{\phi}p/\epsilon)}{\epsilon^2}$,  we have with probability at least $1-\exp(-p)$ over the randomness of $\hat{\mu}$ and $\hat{\nu}$, 
	$$
	|d_{\mathcal{F},\phi}(\hat{\mu}, \hat{\nu}) - d_{\mathcal{F},\phi}(\mu, \nu)| \le \epsilon.
	$$
\end{theorem} 

See Appendix~\ref{sec:app:generalize} for the proof. The intuition is that there aren't too many distinct discriminators, and thus given enough samples the expectation over the empirical distribution converges to the expectation over the true distribution for {\em all} discriminators.

Theorem~\ref{thm:nngeneralizesingle} shows that the neural network divergence (and neural network distance) has a much better generalization properties than Jensen-Shannon divergence or Wasserstein distance. If the GAN successfully minimized the neural network divergence between the empirical distributions, that is, $d(\hatDreal, \hat{\cD}_G)$, then we know the neural network divergence $d(\Dreal, \cD_G)$ between the distributions $\D_{real}$ and $\cD_G$ is also small. It is possible to change the proof to also show that this generalization continues to hold at every iteration of the training as shown in the following corollary.

\begin{corollary}
\label{cor:nngeneralizeall}
In the setting of Theorem~\ref{thm:nngeneralizesingle}, suppose $G^{(1)},G^{(2)},...,G^{(K)}$ ($\log K\ll d$) be the $K$ generators in the $K$ iterations of the training, and assume $\log K \le p$. There is a some universal constant $c$ such that when $m \ge \frac{cp\Delta^2\log (LL_{\phi}p/\epsilon)}{\epsilon^2}$, with probability at least $1-\exp(-p)$, for all $t\in [K]$, 
\begin{align}
\left|d_{\mathcal{F},\phi}(\Dreal, \cD_{G^{(t)}}) -d_{\mathcal{F},\phi}(\hatDreal, \hat{\cD}_{G^{(t)}}) \right|\le \eps\,.
\nonumber
\end{align}
\end{corollary} 

The key observation here is that the objective is separated into two parts and the generator is not directly related to $\D_{real}$. So even though we don't have fresh examples, the generalization bound still holds. Detailed proof appears in Appendix~\ref{sec:app:generalize}.

\subsection{Generalization vs Diversity}\label{sec:diversity} 
Since the final goal of GANs training is to learn a distribution, it is worth understanding that though
weak generalization in the sense of Section~\ref{sec:neural-net-gen} is guaranteed, it comes with a cost. For JS divergence and Wasserstein distance, when the distance between two distributions $\mu,\nu$ is small, it is safe to conclude that the distributions $\mu$ and $\nu$ are almost the same. However, the neural net distance $d_{NN}(\mu,\nu)$ can be small even if $\mu,\nu$ are not very close. As a simple Corollary of Lemma~\ref{thm:nngeneralizesingle}, we obtain:

\begin{corollary}[Low-capacity discriminators cannot detect lack of diversity]\label{corr:lackofdiversity}
	Let $\hat{\mu}$ be the empirical version of distribution $\mu$ with $m$ samples. There is a some universal constant $c$ such that  when $m \ge \frac{cp\Delta^2\log (LL_{\phi}p/\epsilon)}{\epsilon^2}$, we have that with probability at least $1-\exp(-p)$, $d_{\mathcal{F,\phi}}(\mu,\hat{\mu}) \le \epsilon.
	$
\end{corollary}

That is, the neural network distance for nets with $p$ parameters cannot distinguish between a distribution $\mu$ and a  distribution with support $\tilde{O}(p/\epsilon^2)$. In fact the proof still works if the disriminator is allowed to take many more samples from $\mu$; the reason they don't help  is that its  capacity is limited to $p$. 

We note that similar results have been shown before in study of pseudorandomness~\citep{trevisan2009regularity} and model criticism~\citep{gretton2012kernel}.

\section{Expressive power and existence of equilibrium}
\label{sec:infinitemix}

Section~\ref{sec:generalize} clarified the notion of generalization for GANs: namely, neural-net divergence between the generated distribution $\cD$ 
and $\cDr$ on the empirical samples closely tracks the divergence on the full distribution (i.e., unseen samples). But this doesn't explain why in practice the generator usually \textquotedblleft wins\textquotedblright  so that
the discriminator is unable to do much better than random guessing at the end. In other words, was it sheer luck that 
so many real-life distributions $\cDr$  turned out to be close in neural-net distance to a distribution produced by a fairly compact neural net? This section suggests no  luck may be needed.

The explanation starts with a thought experiment. Imagine allowing a much more powerful generator, namely, an infinite mixture of deep nets, each of size $p$.
So long as the deep net class is capable of generating simple gaussians, such mixtures are quite powerful, since a classical result says that an infinite mixtures of simple gaussians can closely approximate $\cDr$. Thus  an infinite mixture of deep net generators will \textquotedblleft win\textquotedblright\ the GAN game, not only against a  discriminator that is a small deep net but also 
against more powerful discriminators (e.g., any Lipschitz function). 

The next stage in the thought experiment is to imagine a much less powerful generator, which is a mix of only a few deep nets, not infinitely many. Simple counterexamples show that now the distribution $\cD$ will not closely approximate arbitrary $\cDr$ with respect to natural metrics like $\ell_p$. Nevertheless, could the generator still win the GAN game against a deep net of bounded capacity (i.e., the deep net is unable to distinguish $\cD$ and $\cDr$)? We show it can.

{\sc informal theorem:} {\em If the discriminator is a deep net with $p$ parameters, then a mixture of $\tilde{O}(p\log (p/\epsilon)/\epsilon^2)$ generator nets can produce a distribution $\cD$ that the discriminator will be unable to distinguish from $\cDr$ with probability more than $\epsilon$. (Here $\tilde{O}(\cdot)$ notation hides some nuisance factors.)} 

This informal theorem  is also a component of our result below about the existence of an approximate pure equilibrium. We will first show that a finite mixture of generators can ``win'' against all discriminators, and then discuss how this mixed generator can be realized as a single generator network that is 1-layer deeper.

%With current technique this existence result seems sensitive to the measuring function $\phi$, and works for
%$\phi(x) =x$ (i.e., Wasserstein GAN). For other $\phi$ we only show existence of mixed equilibria with small mixtures.

\subsection{Equilibrium using a Mixture of Generators}

%For general measuring function $\phi$ we can only show the existence of a mixed equilibrium, where we allow the discriminator and generator to be finite mixtures of deep nets.

For a class of generators $\{G_u,u\in \U\}$ and a class of discriminators $\{D_v,v\in \V\}$, we can define the payoff $F(u,v)$ of the  game between generator and discriminator
\begin{equation}\label{eq:payoff}
F(u,v) = \E_{x\sim \D_{real}}[\phi(D_v(x))] + \E_{x\sim \cD_G}[\phi(1-D_v(x)))].
\end{equation}
Of course as we discussed in previous section, in practice these expectations should be with respect to the empirical distributions. Our discussions in this section does not depend on the distributions $\D_{real}$ and $\D_h$, so we define $F(u,v)$ this way for simplicity.

The well-known min-max theorem~\citep{v1928theorie} in game theory shows if both players are allowed to play {\em mixed strategies} then the game has a min-max solution. A mixed strategy for the generator is just a distribution $\S_u$ supported on $\U$, and one for discriminator is a distribution $S_v$ supported on $\V$.

\begin{theorem}[von Neumann]\label{thm:minmax}
There exists a value $V$, and a pair of mixed strategies ($\mathcal{S}_u$,$\mathcal{S}_v$) such that
$$\forall v, \quad \E_{u\sim \mathcal{S}_u}[F(u,v)] \le V\quad\mbox{and}\quad \forall u, \quad \E_{v\sim \mathcal{S}_v}[F(u,v)] \ge V.$$
\end{theorem}

Note that this equilibrium involves both parties announcing their strategies $\S_u, \S_v$ at the start, such that neither will have any incentive to change their strategy after studying the opponent's strategy. The payoff is generated by the generator  first sample $u\sim \S_u, h\sim \D_h$, and then generate an example $x = G_u(h)$. Therefore, the mixed generator is just a linear mixture of generators. A mixture of discriminators is more complicated because the objective function need not be linear in the discriminator. However in the case of our interest, the generator wins and even a mixture of discriminators cannot effectively distinguish between generated and real distribution. Therefore we do not consider a mixture of discriminators here.
%The discriminator will first sample $v\sim \S_v$, and then output $D_v(x)$. Note that in general this is very different from a discriminator $D$ that outputs $\E_{v\sim \S_v}[D_v(x)]$, because the measuring function $\phi$ is in general nonlinear. In particular, the correct payoff function for a mixture of discriminator is:
%
%\begin{align*}
%\E_{v\sim\S_v}[F(u,v)]  = \E_{\substack{x\sim \D_{real}\\v\sim\S_v}}[\phi(D_v(x))]+ \E_{\substack{h\sim \D_h\\v\sim \S_v}}[\phi(1-D_v(G_u(h)))].
%\end{align*}
%

Of course, this equilibrium involving an infinite mixture makes little sense in practice. We show that (as is folklore in game theory~\citep{lipton1994simple}) that we can {\em approximate} this min-max solution with mixture of finitely many generators and discriminators. More precisely we define $\epsilon$-approximate equilibrium:

\begin{definition}
A pair of mixed strategies $(\S_u,\S_v)$ is an $\epsilon$-approximate equilibrium, if for some value $V$
\begin{align*}
\forall v\in \mathcal{V}, &\quad \E_{u\sim\S_u}[F(u,v)] \le V+\epsilon;\\
\forall u\in \mathcal{U}, &\quad \E_{v\sim\S_v}[F(u,v)] \ge V-\epsilon.
\end{align*}
If the strategies $\S_u,\S_v$ are pure strategies, then this pair is called an $\epsilon$-approximate pure equilibrium.
\end{definition}

Suppose $\phi$ is $L_{\phi}$-Lipschitz and bounded in $[-\Delta,\Delta]$, the generator and discriminators are $L$-Lipschitz with respect to the parameters and $L'$-Lipschitz with respect to inputs, in this setting we can formalize the above Informal Theorem as follows:

\begin{theorem}\label{thm:mixedequilibrium}
In the settings above, if the generator can approximate any point mass\footnote{For all points $x$ and any $\epsilon>0$, there is a generator such that $\E_{h\sim\D_h}[\|G(h)-x\|] \le \epsilon$.}, there is a universal constant $C>0$ such that for any $\epsilon$, there exists $T = \frac{C\Delta^2 p\log (L L' L_{\phi}\cdot p/\epsilon)}{\epsilon^2}$ generators $G_{u_1},\ldots G_{u_T}$. % and $T$ discriminators $D_{v_1},\ldots D_{v_T}$, 
Let $\S_u$ be a uniform distribution on $u_i$, and $D$ is a discriminator that outputs only $1/2$, %and $\S_v$ be a uniform distribution on $v_i$, 
then $(\S_u,D)$ is an $\epsilon$-approximate equilibrium. %Furthermore, in this equilibrium the generator \textquotedblleft wins,\textquotedblright
%meaning discriminators cannot do better than random guessing.
\end{theorem}

The proof uses a standard probabilistic argument and epsilon net argument to show that if we sample $T$ generators and discriminators from infinite mixture, they form an approximate equilibrium with high probability. For the second part, we use the fact that the generator can approximate any point mass, so an infinite mixture of generators can approximate the real distribution $\D_{real}$ to win. Therefore indeed a mixture of $\tilde{O}(p)$ generators can achieve an $\epsilon$-approximate equilibrium.

Note that this theorem works for a wide class of measuring functions $\phi$ (as long as $\phi$ is concave). The generator always wins, and the discriminator's (near) optimal strategy corresponds to random guessing (output a constant $1/2$).

\subsection{Achieving Pure Equilibrium}%$\phi(x)=x$: Pure Equilibrium}

%Now we consider the special case of Wasserstein-GAN, which is equivalent to setting $\phi(x) = x$ in Equation~\eqref{eq:gan-general}. In this case,  

Now we give a construction to augment the network structure, and achieve an approximate pure equilibrium for the GAN game for generator nets of size $\tilde{O}(p^2)$. This should be interpreted as: if deep nets of size $p$ are capable of generating any point mass, then the GAN game for the generator neural network of size $\tilde{O}(p^2)$ has an approximate equilibrium in which the generator wins. (The theorem is stated for RELU gates but also holds for standard activations  such as sigmoid.)

\begin{theorem}\label{thm:pureequilibrium}
Suppose the generator and discriminator are both $k$-layer neural networks ($k\ge 2$) with $p$ parameters, and the last layer uses ReLU activation function. In the setting of Theorem~\ref{thm:mixedequilibrium}, there exists $k+1$-layer neural networks of generators $G$ and discriminator $D$ with $O\left(\frac{\Delta^2 p^2\log (L L' L_{\phi}\cdot p/\epsilon)}{\epsilon^2}\right)$ parameters, such that there exists an $\epsilon$-approximate pure equilibrium with value $2\phi(1/2)$.
\end{theorem}

To prove this theorem, we consider the mixture of generators as in Theorem~\ref{thm:mixedequilibrium}, and show how to fold the mixture into a  larger $k+1$-layer neural network. We sketch the idea; details are in the Appendix~\ref{sec:app:section4}. 
%\paragraph{Mixture of Discriminators} Given a mixture of $T$ discriminators $D_{v_1},...,D_{v_T}$,  we can just define $D(x) = \frac{1}{T} \sum_{i=1}^T D_{v_i}(x)$. Because $\phi(x) = x$, we know for any generator $G_u$
%\begin{align*}
%\E_{x\sim \D_{real}}[D(x)] & = \E_{x\sim \D_{real},i\in[T]}[D_{v_i}(x)] \\
%\E_{h\sim \D_h}[(1-D(G_u(h))]& = \E_{h\sim \D_h,i\in[T]}[(1-D_{v_i}(G_u(h))].
%\end{align*}
%
%Therefore, the payoff for $D$ is exactly the same as the payoff for the mixture of discriminators. Also, the discriminator $D$ is  easily implemented as a network $T$ times as large as the original network by adding a top layer that averages the output of the $T$ generators $D_{v_i}(x)$. 

%\paragraph{Mixture of Generators} 
For mixture of generators, we construct a single neural network that approximately generates the mixture distribution using the gaussian input it has. To do that, we can pass the input $h$ through all the generators $G_{u_1}, G_{u_2},..., G_{u_T}$. We then show how to implement a ``multi-way selector'' that will select a uniformly random output from $\cD_{G_{u_i}}$ ($i\in [T]$). The selector involves a simple $2$-layer network that selects a number $i$ from $1$ to $T$ with the appropriate probability and \textquotedblleft disables\textquotedblright all the neural nets except the $i$th one by forwarding an appropriate large negative input.

%Theorem~\ref{thm:pureequilibrium} suggests that the objective of Wasserstein GAN may have approximate pure equilibrium for certain architectures of neural networks, which lends credence that Wasserstein GAN may be more robust.

{\em Remark:} In practice, GANs use highly structured deep nets, such as convolutional nets. Our current proof of existence of  pure equilibrium requires introducing less structured elements in the net, namely, the multiway selectors that implement the mixture within a single net.  It is left for future work whether pure equilibria exist for the original structured architectures.  In the meantime, in practice we recommend using, even for W-GAN, a mixture of structured nets for GAN training, and it seems to help in our experiments reported below.

\section{MIX+GANs}

Theorem~\ref{thm:mixedequilibrium} and Theorem~\ref{thm:pureequilibrium} show that using a mixture of (not too many) generators and discriminators guarantees existence of approximate equilibrium. This suggests that using a mixture may lead to more  stable training. Our experiments correspond to an older version of this paper, and they are done using a mixture for both generator and discriminators.

Of course, it is impractical to use very large mixtures, so we propose {\sc mix + gan}: use a mixture of $T$ components, where $T$ is as large as allowed by size of GPU memory (usually $T \leq 5$). Namely, train a mixture of $T$ generators $\{G_{u_i}, i\in [T]\}$ and $T$ discriminators $\{D_{v_i}, i\in [T]\}$) which share the same network architecture but have their own trainable parameters. Maintaining a mixture means of course maintaining a  weight $w_{u_i}$ for the generator $G_{u_i}$ which corresponds to the probability of selecting the output of $G_{u_i}$. These weights are also updated via backpropagation.  This heuristic can be combined with existing
methods like {\sc dcgan, w-gan} etc., giving us new training methods {\sc mix+dcgan, mix+w-gan} etc.

We use exponentiated gradient~\citep{kivinen1997exponentiated}: store the log-probabilities $\{\alpha_{u_i}, i\in[T]\}$, and then obtain the weights by applying soft-max function on them: 
$$w_{u_i}=\frac{e^{\alpha_{u_i}}}{\sum_{k=1}^T e^{\alpha_{u_k}}},~~~i\in[T]$$

Note that our algorithm is maintaining weights on different generators and discriminators. This is very different from the idea of {\em boosting} where weights are maintained on samples. AdaGAN~\citep{tolstikhin2017adagan} uses ideas similar to boosting and maintains weights on training examples.

Given payoff function $F$, training  {\sc mix + gan} boils down to optimizing:
\begin{align*}
& \min_{\{u_i\},\{\alpha_{u_i}\}}\max_{\{v_j\},\{\alpha_{v_j}\}}\E_{i,j\in[T]}F(u_i,v_j) \\
& = \min_{\{u_i\},\{\alpha_{u_i}\}}\max_{\{v_j\},\{\alpha_{v_j}\}}\sum_{i,j\in[T]}w_{u_i}w_{v_j}F(u_i,v_j).
\end{align*}
Here the payoff function is the same as Equation \eqref{eq:payoff}. We use both measuring functions $\phi(x) = \log x$ (for original GAN) and $\phi(x)=x$ (for WassersteinGAN).
In our experiments we alternatively update generators' and discriminators' parameters as well as their corresponding log-probabilities using ADAM~\citep{2015adam}, with learning rate $lr=0.0001$.

Empirically, it is observed that some components of the mixture tend to collapse and their weights diminish during the training. To encourage full use of the mixture capacity, we add to the training objective an entropy regularizer term that discourages the weights being too far away from uniform:
\begin{align*}
& R_{ent}(\{w_{u_i}\}, \{w_{v_i}\}) = -\frac{1}{T}\sum_{i=1}^T(\log(w_{u_i}) + \log(w_{v_i}))
\end{align*}

\section{Experiments}  \label{sec:experiments}

\begin{figure}[h]
	\centering
\includegraphics[width=0.7\columnwidth]{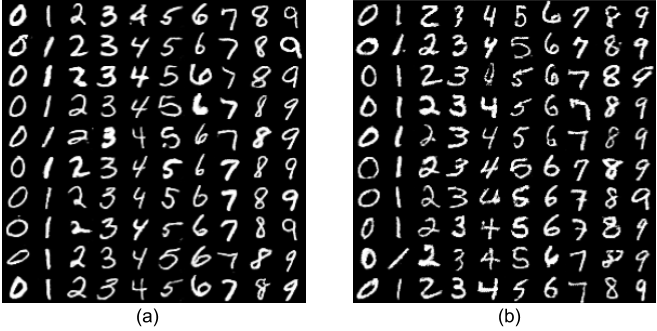}
	\caption{\textbf{MNIST Samples}. Digits generated from (a) MIX+DCGAN and (b) DCGAN. }%(c)-(d)-(e) From each of $3$ components of  MIX+DCGAN. (View on-screen with zooming in.)}
	\label{fig:mnist_results}
\end{figure}
\begin{figure}[h]
	\centering
	\includegraphics[width=0.7\columnwidth]{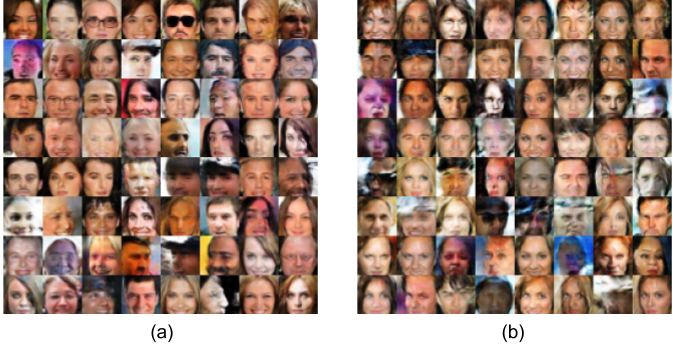}
	\caption{\textbf{CelebA Samples}. Faces generated from (a) \MDC and (b) \DCGAN.}
%	\includegraphics[width=0.9\columnwidth]{celebA_result.pdf} %
%	\caption{\textbf{CelebA Samples}. Faces generated from (a) MIX+DCGAN. (b) DCGAN.(c)-(d)-(e) Each of  $3$ components of MIX+DCGAN. (View on-screen with zooming in.)}
	\label{fig:celebA_results}
\end{figure}

\begin{table}[h]
	\caption{\textbf{Inception Scores on CIFAR-10.} Mixture of DCGANs achieves higher score than any single-component DCGAN does. All models except for WassersteinGAN variants are trained with labels.}
	\label{tab:inception_score}
	\begin{center}
		\begin{small}
			\begin{tabular}{lr}
				\hline
				Method & Score \\
				\hline
				SteinGAN~\citep{Dilin2016Stein} & 6.35\\
				Improved GAN~\citep{Salimans2016ImprovedGANs} & 8.09$\pm$0.07\\
				AC-GAN~\citep{odena2016conditional} &  8.25 $\pm$ 0.07\\
				S-GAN (best variant in~\citep{xh2016SGAN}) & 8.59$\pm$ 0.12\\
				DCGAN (as reported in~\citet{Dilin2016Stein}) & 6.58\\
				DCGAN (best variant in~\citet{xh2016SGAN}) & 7.16$\pm$0.10\\
				DCGAN (5x size) & 7.34$\pm$0.07\\
							MIX+DCGAN (Ours, with $5$ components) & 7.72$\pm$0.09 \\
				\hline
								Wasserstein GAN & 3.82$\pm$0.06\\
				MIX+WassersteinGAN (Ours, with $5$ components) & 4.04$\pm$0.07\\
				\hline
								Real data & 11.24$\pm$0.12
			\end{tabular}
		\end{small}
	\end{center}
\end{table}

\begin{figure}[h]
	\centering
	\includegraphics[width=0.7\columnwidth]{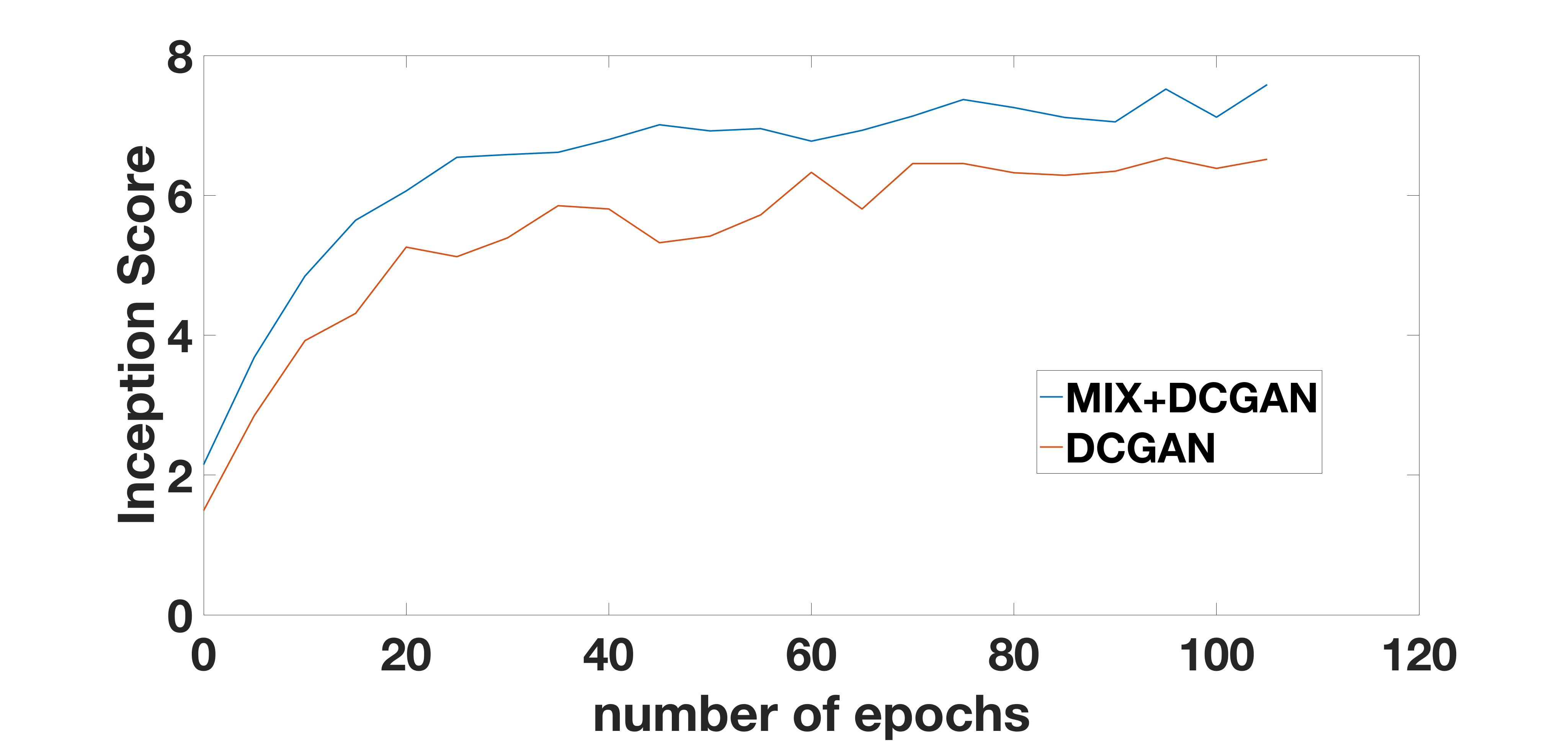}
	\caption{\textbf{MIX+DCGAN v.s. DCGAN Training Curve (Inception Score)}. MIX+DCGAN is consistently higher than DCGAN.}
	\label{fig:mix_inception_score_curve}
	\includegraphics[width=0.7\columnwidth]{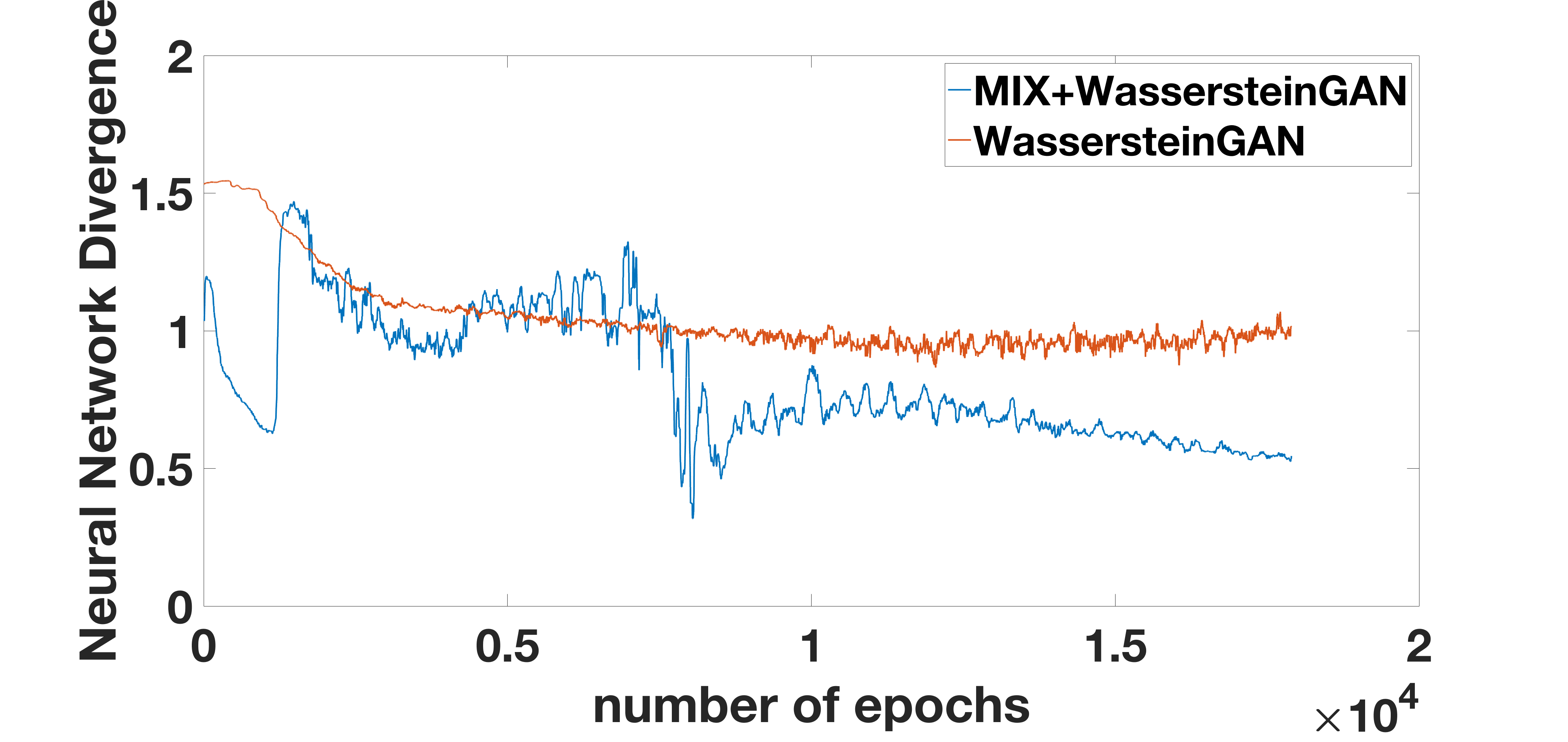}
	\caption{\textbf{MIX+WassersteinGAN v.s. WassersteinGAN Training Curve (Wasserstein Objective)}.  MIX+WassersteinGAN  is better towards the end but loss drops less smoothly, which needs further investigation.}
	\label{fig:inception_score_wgan_curve}
\end{figure}
In this section, we  first explore  the qualitative benefits of our method on image generation tasks: MNIST dataset~\citep{lecun1998mnist} of hand-written digits and the CelebA~\citep{liu2015faceattributes} dataset of human faces. Then for more quantitative evaluation we use the CIFAR-10 dataset~\citep{alex09} and use the {\em Inception Score} introduced in~\cite{Salimans2016ImprovedGANs}.  MNIST contains 60,000 labeled 28$\times$28-sized images of hand-written digits, CelebA contains over 200K 108$\times$108-sized images of human faces (we crop the center 64$\times$64 pixels for our experiments), and CIFAR-10 has 60,000 labeled 32$\times$32-sized RGB natural images which fall into 10 categories. 

To reinforce the point that this technique works out of the box, no extensive hyper-parameter search or tuning is necessary. Please refer to our code for experimental setup.\footnote{Related code is public online at \url{https://github.com/PrincetonML/MIX-plus-GANs.git}}

\subsection{Qualitative Results}

The DCGAN architecture~\citep{soumith2016DCGAN} uses deep convolutional nets as generators and discriminators. We trained {\sc mix + dcgan} on MNIST and CelebA using the authors' code as a black box, and compared visual qualities of generated images vs DCGAN.

The DCGAN architecture~\citep{soumith2016DCGAN} uses deep convolutional nets as generators and discriminators. We trained \MDC on MNIST and CelebA using the authors' code as a black box, and compared visual qualities of generated images to those by \DCGAN.

Results on MNIST is shown in Figure~\ref{fig:mnist_results}. In this experiment, the baseline \DCGAN consists of a pair of a generator and a discriminator, which are 5-layer deconvoluitonal neural networks, and are conditioned on image labels. Our \MDC model consists of a mixture of such DCGANs so that it has $3$ generators and $3$ discriminators. 
We observe that our method produces somewhat cleaner digits than the baseline (note the fuzziness in the latter).

Results on CelebA dataset are also in Figure~\ref{fig:celebA_results}, using the same architecture as for MNIST, except the models are not conditioned on image labels anymore. Again, our method generates more faithful and more diverse samples than the baseline. Note that one may need to zoom in to fully perceive the difference, since both the two datasets are rather easy for \DCGAN.

In Appendix~\ref{sec:app:components}, we also show generated digits and faces from each components of \MDC.

\subsection{Quantitative Results} 

Now we turn to quantitative measurement using Inception Score~\citep{Salimans2016ImprovedGANs}. Our method is applied to \DCGAN and \WGAN~\cite{arjovsky2017wasserstein}, and throughout, mixtures of $5$ generators and $5$ discriminators are used. 
At first sight the comparison \DCGAN v.s.\ \MDC seems unfair because the latter uses $5$ times the capacity of the former, with corresponding penalty in running time per epoch. To address this, we also compare our method with larger versions of \DCGAN with roughly the same number of parameters, and we found the former is consistently better than the later, as detailed below. 

To construct \MDC,  we build on top of the DCGAN trained with losses proposed by~\citet{xh2016SGAN}, which is the best variant so far without improved training techniques. The same hyper-parameters are used for fair comparison. See~\cite{xh2016SGAN} for more details. Similarly, for the \MW, the base GAN is identical to that proposed by~\citet{arjovsky2017wasserstein} using their hyper-parameter scheme. 

For a quantitative comparison, inception score is calculated for each model, using 50,000 freshly generated samples that are not used in training. To sample a single image from our MIX+ models, we first select a generator from the mixture according to their assigned weights $\{w_{u_i}\}$, and then draw a sample from the selected generator. 

Table~\ref{tab:inception_score} shows the results on the CIFAR-10 dataset.
We find that, simply by applying our method to the baseline models, our MIX+ models achieve 7.72 v.s.\ 7.16 on \DCGAN, and 4.04 v.s.\ 3.82 on \WGAN. To confirm that the superiority of MIX+ models is not solely due to more parameters, we also tested a \DCGAN model with 5 times many parameters (roughly the same number of parameters as a 5-component \MDC), which is tuned using a grid search over 27 sets of hyper-parameters (learning rates, dropout rates, and regularization weights). It gets only 7.34 (labeled as "5x size" in Table~\ref{tab:inception_score}), which is lower than that of a \MDC. It is unclear how to apply MIX+ to S-GANs. We tried mixtures of the upper and bottom generators separately, resulting in worse scores somehow. We leave that for future exploration.

Figure~\ref{fig:mix_inception_score_curve} shows how Inception Scores of MIX+DCGAN v.s. DCGAN evolve during training. MIX+DCGAN outperforms DCGAN throughout the entire training process, showing that it makes effective use of the additional capacity.

\citet{arjovsky2017wasserstein} shows that (approximated) Wasserstein loss, which is the neural network divergence by our definition, is meaningful because it correlates well with visual quality of generated samples. Figure~\ref{fig:inception_score_wgan_curve} shows the training dynamics of neural network divergence of MIX+WassersteinGAN v.s. WassersteinGAN, which strongly indicates that MIX+WassersteinGAN is capable of achieving a much lower divergence as well as of improving the visual quality of generated samples.

\section{Conclusions}
The notion of generalization for GANs has been clarified by introducing a new notion of distance between distributions, the neural net distance. (Whereas popular distances such as Wasserstein and JS may not generalize.) Assuming the visual cortex also is a deep net (or some network of moderate capacity) generalization with respect to this metric is in principle sufficient to make the final samples look realistic to humans, even if the GAN doesn't actually learn the true distribution.

One issue raised by our analysis is that the current GANs objectives cannot even enforce that the synthetic distribution has high diversity (Section~\ref{sec:diversity}). This is empirically verified in a follow-up work\citep{arora2017gans}. Furthermore the issue cannot be fixed by simply providing the discriminator with more training examples. Possibly some other change to the GANs setup are needed. 

The paper also made progress another unexplained issue about GANs, by showing that a pure approximate equilibrium exists for a certain natural training objective (Wasserstein) and in which the generator wins the game. No assumption about the target distribution $\cDr$ is needed.

Suspecting that a pure equilibrium may not exist for all objectives,  we recommend in practice our \MG protocol using a small mixture of discriminators and generators. Our experiments show it improves the quality of several existing GAN training methods.

Finally, existence of an equilibrium does not imply that a simple algorithm (in this case, backpropagation) would find it easily. Understanding convergence remains widely open.

\section*{Acknowledgements} 

This paper was done in part while the authors were hosted by Simons Institute. We thank Moritz Hardt, Kunal Talwar, Luca Trevisan, Eric Price, and the referees for useful comments. This research was supported by NSF, Office of Naval Research, and the Simons Foundation.

\clearpage

\bibliographystyle{plainnat}
\bibliography{gan}

\begin{thebibliography}{24}
\providecommand{\natexlab}[1]{#1}
\providecommand{\url}[1]{\texttt{#1}}
\expandafter\ifx\csname urlstyle\endcsname\relax
  \providecommand{\doi}[1]{doi: #1}\else
  \providecommand{\doi}{doi: \begingroup \urlstyle{rm}\Url}\fi

\bibitem[Abadi and Andersen(2016)]{abadi2016learning}
Mart{\'\i}n Abadi and David~G Andersen.
\newblock Learning to protect communications with adversarial neural
  cryptography.
\newblock \emph{arXiv preprint arXiv:1610.06918}, 2016.

\bibitem[Arjovsky et~al.(2017)Arjovsky, Chintala, and
  Bottou]{arjovsky2017wasserstein}
Martin Arjovsky, Soumith Chintala, and L{\'e}on Bottou.
\newblock Wasserstein gan.
\newblock \emph{arXiv preprint arXiv:1701.07875}, 2017.

\bibitem[Arora and Zhang(2017)]{arora2017gans}
Sanjeev Arora and Yi~Zhang.
\newblock Do gans actually learn the distribution? an empirical study.
\newblock \emph{arXiv preprint arXiv:1706.08224}, 2017.

\bibitem[{Durugkar} et~al.(2016){Durugkar}, {Gemp}, and
  {Mahadevan}]{2016arXiv161101673D}
I.~{Durugkar}, I.~{Gemp}, and S.~{Mahadevan}.
\newblock {Generative Multi-Adversarial Networks}.
\newblock \emph{ArXiv e-prints}, November 2016.

\bibitem[Ghosh et~al.(2003)Ghosh, Ghosh, and Ramamoorthi]{ghosh2003bayesian}
Jayanta~K Ghosh, RVJK Ghosh, and RV~Ramamoorthi.
\newblock Bayesian nonparametrics.
\newblock Technical report, 2003.

\bibitem[Goodfellow(2016)]{goodfellow2016nips}
Ian Goodfellow.
\newblock Nips 2016 tutorial: Generative adversarial networks.
\newblock \emph{arXiv preprint arXiv:1701.00160}, 2016.

\bibitem[Goodfellow et~al.(2014)Goodfellow, Pouget-Abadie, Mirza, Xu,
  Warde-Farley, Ozair, Courville, and Bengio]{goodfellow2014generative}
Ian Goodfellow, Jean Pouget-Abadie, Mehdi Mirza, Bing Xu, David Warde-Farley,
  Sherjil Ozair, Aaron Courville, and Yoshua Bengio.
\newblock Generative adversarial nets.
\newblock In \emph{Advances in neural information processing systems}, pages
  2672--2680, 2014.

\bibitem[Gretton et~al.(2012)Gretton, Borgwardt, Rasch, Sch{\"o}lkopf, and
  Smola]{gretton2012kernel}
Arthur Gretton, Karsten~M Borgwardt, Malte~J Rasch, Bernhard Sch{\"o}lkopf, and
  Alexander Smola.
\newblock A kernel two-sample test.
\newblock \emph{Journal of Machine Learning Research}, 13\penalty0
  (Mar):\penalty0 723--773, 2012.

\bibitem[Huang et~al.(2017)Huang, Li, Poursaeed, Hopcroft, and
  Belongie]{xh2016SGAN}
Xun Huang, Yixuan Li, Omid Poursaeed, John Hopcroft, and Serge Belongie.
\newblock Stacked generative adversarial networks.
\newblock In \emph{Computer Vision and Patter Recognition}, 2017.

\bibitem[{Jiwoong Im} et~al.(2016){Jiwoong Im}, {Ma}, {Dongjoo Kim}, and
  {Taylor}]{2016arXiv161204021J}
D.~{Jiwoong Im}, H.~{Ma}, C.~{Dongjoo Kim}, and G.~{Taylor}.
\newblock {Generative Adversarial Parallelization}.
\newblock \emph{ArXiv e-prints}, December 2016.

\bibitem[Kingma and Ba(2015)]{2015adam}
Diederik Kingma and Jimmy Ba.
\newblock Adam: A method for stochastic optimization.
\newblock In \emph{International Conference on Learning Representations}, 2015.

\bibitem[Kivinen and Warmuth(1997)]{kivinen1997exponentiated}
Jyrki Kivinen and Manfred~K Warmuth.
\newblock Exponentiated gradient versus gradient descent for linear predictors.
\newblock \emph{Information and Computation}, 132\penalty0 (1):\penalty0 1--63,
  1997.

\bibitem[Krizhevsky and Hinton(2009)]{alex09}
Alex Krizhevsky and Geoffrey Hinton.
\newblock Learning multiple layers of features from tiny images.
\newblock Technical report, 2009.

\bibitem[LeCun et~al.(1998)LeCun, Cortes, and Burges]{lecun1998mnist}
Yann LeCun, Corinna Cortes, and Christopher~JC Burges.
\newblock The mnist database of handwritten digits, 1998.

\bibitem[Lipton and Young(1994)]{lipton1994simple}
Richard~J Lipton and Neal~E Young.
\newblock Simple strategies for large zero-sum games with applications to
  complexity theory.
\newblock In \emph{Proceedings of the twenty-sixth annual ACM symposium on
  Theory of computing}, pages 734--740. ACM, 1994.

\bibitem[Liu et~al.(2015)Liu, Luo, Wang, and Tang]{liu2015faceattributes}
Ziwei Liu, Ping Luo, Xiaogang Wang, and Xiaoou Tang.
\newblock Deep learning face attributes in the wild.
\newblock In \emph{Proceedings of the IEEE International Conference on Computer
  Vision}, pages 3730--3738, 2015.

\bibitem[M{\"u}ller(1997)]{muller1997integral}
Alfred M{\"u}ller.
\newblock Integral probability metrics and their generating classes of
  functions.
\newblock \emph{Advances in Applied Probability}, 29\penalty0 (02):\penalty0
  429--443, 1997.

\bibitem[Odena et~al.(2016)Odena, Olah, and Shlens]{odena2016conditional}
Augustus Odena, Christopher Olah, and Jonathon Shlens.
\newblock Conditional image synthesis with auxiliary classifier gans.
\newblock \emph{arXiv preprint arXiv:1610.09585}, 2016.

\bibitem[Radford et~al.(2016)Radford, Metz, and Chintala]{soumith2016DCGAN}
Alec Radford, Luke Metz, and Soumith Chintala.
\newblock Unsupervised representation learning with deep convolutional
  generative adversarial networks.
\newblock In \emph{International Conference on Learning Representations}, 2016.

\bibitem[Salimans et~al.(2016)Salimans, Goodfellow, Zaremba, Cheung, Radford,
  and Chen]{Salimans2016ImprovedGANs}
Tim Salimans, Ian Goodfellow, Wojciech Zaremba, Vicki Cheung, Alec Radford, and
  Xi~Chen.
\newblock Improved techniques for training gans.
\newblock In \emph{Advances in Neural Information Processing Systems}, 2016.

\bibitem[Tolstikhin et~al.(2017)Tolstikhin, Gelly, Bousquet, Simon-Gabriel, and
  Sch{\"o}lkopf]{tolstikhin2017adagan}
Ilya Tolstikhin, Sylvain Gelly, Olivier Bousquet, Carl-Johann Simon-Gabriel,
  and Bernhard Sch{\"o}lkopf.
\newblock Adagan: Boosting generative models.
\newblock \emph{arXiv preprint arXiv:1701.02386}, 2017.

\bibitem[Trevisan et~al.(2009)Trevisan, Tulsiani, and
  Vadhan]{trevisan2009regularity}
Luca Trevisan, Madhur Tulsiani, and Salil Vadhan.
\newblock Regularity, boosting, and efficiently simulating every high-entropy
  distribution.
\newblock In \emph{Computational Complexity, 2009. CCC'09. 24th Annual IEEE
  Conference on}, pages 126--136. IEEE, 2009.

\bibitem[v.~Neumann(1928)]{v1928theorie}
J~v.~Neumann.
\newblock Zur theorie der gesellschaftsspiele.
\newblock \emph{Mathematische annalen}, 100\penalty0 (1):\penalty0 295--320,
  1928.

\bibitem[Wang and Liu(2016)]{Dilin2016Stein}
Dilin Wang and Qiang Liu.
\newblock Learning to draw samples: With application to amortized mle for
  generative adversarial learning.
\newblock Technical report, 2016.

\end{thebibliography}

\clearpage
\appendix

\section{Generated Samples from Components of \MDC}\label{sec:app:components}

In Figure~\ref{fig:mnist_components} and Figure~\ref{fig:celeba_components} we showed generated digits and face from different components of \MDC. 
\begin{figure}[h]
	\centering
	\includegraphics[width=0.3\columnwidth]{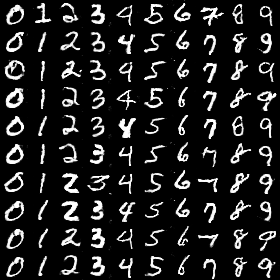}
	\includegraphics[width=0.3\columnwidth]{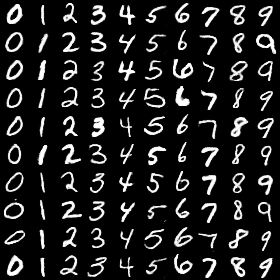}
	\includegraphics[width=0.3\columnwidth]{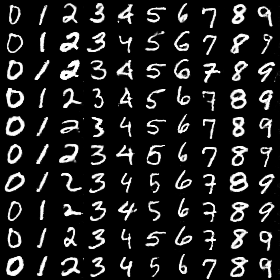}
	\caption{\textbf{MNIST Samples}. Digits generated from each of the 3 components of \MDC}%(c)-(d)-(e) From each of $3$ components of  MIX+DCGAN. (View on-screen with zooming in.)}
	\label{fig:mnist_components}
\end{figure}
\begin{figure}[h]
	\centering
	\includegraphics[width=0.3\columnwidth]{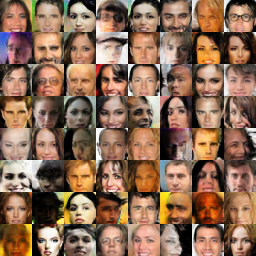}
	\includegraphics[width=0.3\columnwidth]{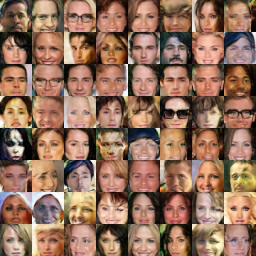}
	\includegraphics[width=0.3\columnwidth]{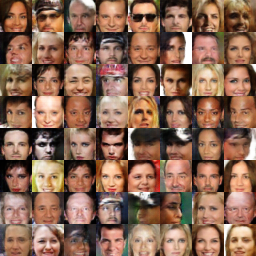}
	\caption{\textbf{MNIST Samples}. Faces generated from each of the 3 components of \MDC}%(c)-(d)-(e) From each of $3$ components of  MIX+DCGAN. (View on-screen with zooming in.)}
	\label{fig:celeba_components}
\end{figure}
\section{Omitted Proofs}

In this section we give detailed proofs for the theorems in the main document.

\subsection{Omitted Proofs for Section~\ref{sec:generalize}}
\label{sec:app:generalize}
We first show that JS divergence and Wasserstein distances can lead to overfitting.

\begin{lemma}[Lemma~\ref{lem:warmup} restated]
	Let $\mu$ be uniform Gaussian distributions $\N(0,\frac{1}{d}I)$ and $\hat{\mu}$ be an empirical versions of $\mu$ with $m$ examples. Then we have 	\begin{align*}
	&d_{JS}(\mu,\hat{\mu}) = \log 2\\
	& d_{W}(\mu,\hat{\mu}) \ge 1.1
	\end{align*}
\end{lemma}
\begin{proof}For Jensen-Shannon divergence, observe that $\mu$ is a continuous distribution and $\hat{\mu}$ is discrete, therefore $d_{JS}(\mu,\hat{\mu}) = \log 2$. 

For Wasserstein distance, let $x_1,x_2,...,x_m$ be the empirical samples (fixed arbitrarily). For $y\sim \N(0,\frac{1}{d}I)$, by standard concentration and union bounds, we have
$$
\Pr[\forall i\in[m] \|y-x_i\| \ge 1.2] \ge 1-m\exp(-\Omega(d)) \ge 1-o(1).
$$
 Therefore, using the earth-mover interpretation of Wasserstein distance, we know $d_{W}(\mu,\hat{\mu}) \ge 1.2\Pr[\forall i\in[m] \|y-x_i\| \ge 1.2] \ge 1.1$.
\end{proof}

Next we consider sampling for both the generated distribution and the real distribution, and show that the JS divergence or Wasserstein distance do not generalize.

\begin{theorem}\label{thm:formal_generalize}
Let $\mu,\nu$ be uniform Gaussian distributions $\N(0,\frac{1}{d}I)$. Suppose $\hat{\mu},\hat{\nu}$ are empirical versions of $\mu,\nu$ with $m$ samples. Then with probability at least $1-m^2\exp(-\Omega(d))$ we have
\begin{align*}
d_{JS}(\mu,\nu) = 0, & d_{JS}(\hat{\mu},\hat{\nu}) = \log 2.\\
d_{W}(\mu,\nu) = 0, & d_{W}(\hat{\mu},\hat{\nu}) \ge 1.1.
\end{align*}
Further, let $\tilde{\mu},\tilde{\nu}$ be the convolution of $\hat{\mu},\hat{\nu}$ with a Gaussian distribution $N(0,\frac{\sigma^2}{d}I)$, as long as $\sigma < \frac{c}{\sqrt{\log m}}$ for small enough constant $c$, we have with probability at least $1-m^2\exp(-\Omega(d))$.
$$
d_{JS}(\tilde{\mu},\tilde{\nu}) > \log 2 - 1/m.
$$

\end{theorem}
\begin{proof}
For the Jensen-Shannon divergence, we know with probability 1 the supports of $\hat{\mu},\hat{\nu}$ are disjoint, therefore $d_{JS}(\hat{\mu},\hat{\nu}) = 1$.

For Wasserstein distance, note that for two random Gaussian vectors $x,y \sim N(0,\frac{1}{d}I)$, their difference is also a Gaussian with expected square norm 2. Therefore we have
$$
\Pr[\|x-y\|^2 \le 2 - \epsilon] \le \exp(-\Omega(\epsilon^2d)).
$$

As a result, setting $\epsilon$ to be a fixed constant (0.1 suffices), with probability $1-m^2\exp(-\Omega(d))$, we can union bound over all the $m^2$ pairwise distances for points in support of $\hat{\mu}$ and support of $\hat{\nu}$. With high probability, the closest pair between $\hat{\mu}$ and $\hat{\nu}$ has distance at least 1, therefore the Wasserstein distance $d_W(\hat{\mu},\hat{\nu}) \ge 1.1$.

Finally we prove that even if we add noise to the two distributions, the JS divergence is still large. For distributions $\tilde{\mu},\tilde{\nu}$, let $\rho_1,\rho_2$ be their density functions. Let $g(x) = \rho_1(x)\log \frac{2\rho_1(x)}{\rho_1(x)+\rho_2(x)} + \rho_2(x)\log \frac{2\rho_2(x)}{\rho_1(x)+\rho_2(x)}$, we can rewrite the JS divergence as
$$
d_{JS}(\tilde{\mu},\tilde{\nu}) = \int \frac{1}{2}g(x) dx.
$$
Let $z_x$ be a Bernoulli variable with probability $\rho_1(x)/(\rho_1(x)+\rho_2(x))$ of being 1. Note that $g(x)=(\rho_1(x)+\rho_2(x))(\log 2 - H(z_x))$ where $H(z_x)$ is the entropy of $z_x$. Therefore $0 \le g(x)\le (\rho_1(x)+\rho_2(x))\log 2$.  Let $\mathcal{X}$ be the union of radius-0.2 balls near the $2m$ samples in $\hat{\mu}$ and $\hat{\nu}$. Since with high probability, all these samples have pairwise distance at least 1, by Gaussian density function we know (a) the balls do not intersect; (b) within each ball $\frac{\max\{d_1(x),d_2(x)\}}{\min\{d_1(x),d_2(x)\}} \ge m^2$; (c) the union of these balls take at least $1-1/2m$ fraction of the density in $(\hat{\mu}+\hat{\nu})/2$.

Therefore for every $x\in \X$, we know $H(z_x) \le o(1/m)$, therefore
\begin{align*}
d_{JS}(\tilde{\mu},\tilde{\nu}) & = \int \frac{1}{2}g(x) dx\\
& \ge\int_{x\in \X} \frac{1}{2}g(x) dx \\
& \ge\int_{x\in \X} (\rho_1(x)+\rho_2(x))(\log 2-o(1/m)) dx\\
& \ge \log 2 - 1/2m-o(1/m) \ge \log 2-1/m.
\end{align*}
\end{proof}

Next we prove the neural network distance does generalize, given enough samples. Let us first recall the settings here: we assume that the measuring function takes values in $[-\Delta,\Delta]$ is $L_\phi$-Lipschitz. Further, $\mathcal{F} = \{D_v, v\in \mathcal{V}\}$ is the class of discriminators that is $L$-Lipschitz with respect to the parameters $v$. As usual, we use $p$ to denote the number of parameters in $v$. 

\begin{theorem}
[Theorem~\ref{thm:nngeneralizesingle} restated]
	In the setting described in the previous paragraph, let $\mu,\nu$ be two distributions and $\hat{\mu},\hat{\nu}$ be empirical versions with at least $m$ samples each. There is a universal constant $c$ such that  when $m \ge \frac{cp\Delta^2\log (LL_{\phi}p/\epsilon)}{\epsilon^2}$,  we have with probability at least $1-\exp(-p)$ over the randomness of $\hat{\mu}$ and $\hat{\nu}$, 
	$$
	|d_{\mathcal{F},\phi}(\hat{\mu}, \hat{\nu}) - d_{\mathcal{F},\phi}(\mu, \nu)| \le \epsilon.
	$$
\end{theorem} 
\begin{proof}
The proof uses concentration bounds. We show that with high probability, for every discriminator $D_v$,
\begin{align}
|\E_{x\sim \mu}[\phi(D_v(x))] - \E_{x\sim \hat{\mu}}[\phi(D_v(x))]| & \le \epsilon/2,\label{eq:bound1}\\
|\E_{x\sim \nu}[\phi(1 - D_v(x))] - \E_{x\sim \hat{\nu}}[\phi(1-D_v(x))]| & \le \epsilon/2.\label{eq:bound2}
\end{align}
If $d_{\mathcal{F},\phi}(\mu, \nu) = t$, let $D_v$ be the optimal discriminator, we then have
\begin{align*}
d_{\mathcal{F},\phi}(\mu, \nu) & \ge \E_{x\sim \hat{\mu}}[\phi(D_v(x))] + \E_{x\sim \hat{\nu}}[\phi(D_v(x))].\\
& \ge \E_{x\sim \mu}[\phi(D_v(x))] + \E_{x\sim \nu}[\phi(D_v(x))] \\ & \quad - |\E_{x\sim \mu}[\phi(D_v(x))] - \E_{x\sim \hat{\mu}}[\phi(D_v(x))]| \\ &\quad- |\E_{x\sim \nu}[\phi(1 - D_v(x))] - \E_{x\sim \hat{\nu}}[\phi(1-D_v(x))]| \\
& \ge t - \epsilon.
\end{align*}
The other direction is similar.

Now we prove the claimed bounds \eqref{eq:bound1} (proof of \eqref{eq:bound2} is identical). Let $\X$ be a finite set such that every point in $\V$ is within distance $\epsilon/8LL_{\phi}$ of a point in $X$ (a so-called $\epsilon/8LL_{\phi}$-net).  Standard constructions give an $X$ satisfying $\log |\X| \le O(p\log (LL_{\phi}p/\epsilon))$. For every $v\in \X$, by Chernoff bound we know
$$
\Pr[|\E_{x\sim \mu}[\phi(D_v(x))] - \E_{x\sim \hat{\mu}}[\phi(D_v(x))]| \ge \frac{\epsilon}{4}] \le 2\exp(-\frac{\epsilon^2 m}{2\Delta^2}).
$$
Therefore, when $m \ge \frac{Cp\Delta^2\log (LL_{\phi}p/\epsilon)}{\epsilon^2}$ for large enough constant $C$, we can union bound over all $v\in \X$. With high probability (at least $1-\exp(-p)$), for all $v\in \X$ we have $|\E_{x\sim \mu}[\phi(D_v(x))] - \E_{x\sim \hat{\mu}}[\phi(D_v(x))]| \ge \frac{\epsilon}{4}$.

Now, for every $v\in \V$, we can find a $v'\in \X$ such that $\|v-v'\| \le \epsilon/8LL_{\phi}$. Therefore
\begin{align*}
&|\E_{x\sim \mu}[\phi(D_v(x))] - \E_{x\sim \hat{\mu}}[\phi(D_v(x))]| \\
\le & |\E_{x\sim \mu}[\phi(D_{v'}(x))] - \E_{x\sim \hat{\mu}}[\phi(D_{v'}(x))]|\\
& + |\E_{x\sim \mu}[\phi(D_{v'}(x))] - \E_{x\sim \mu}[\phi(D_{v}(x))]| \\
& + |\E_{x\sim \hat{\mu}}[\phi(D_{v'}(x))] - \E_{x\sim \hat{\mu}}[\phi(D_{v}(x))]|\\
\le &\epsilon/4+\epsilon/8+\epsilon/8 \\
\le &\epsilon/2.
\end{align*}
This finishes the proof of \eqref{eq:bound1}.
\end{proof}

Finally, we generalize the above Theorem to hold for {\em all} generators in a family.

\begin{corollary}
[Corollary~\ref{cor:nngeneralizeall} restated]
	In the setting of Theorem~\ref{thm:nngeneralizesingle}, suppose $G^{(1)},G^{(2)},...,G^{(K)}$ ($\log K\ll d$) be the $K$ generators in the $K$ iterations of the training, and assume $\log K \le p$. There is a some universal constant $c$ such that when $m \ge \frac{cp\Delta^2\log (LL_{\phi}p/\epsilon)}{\epsilon^2}$, with probability at least $1-\exp(-p)$, for all $t\in [K]$, 
	\begin{align}
	\left|d_{\mathcal{F},\phi}(\Dreal, \cD_{G^{(t)}}) -d_{\mathcal{F},\phi}(\hatDreal, \hat{\cD}_{G^{(t)}}) \right|\le \eps\,.
	\nonumber
	\end{align}
\end{corollary}

\begin{proof}
This follows from the proof of Theorem~\ref{thm:nngeneralizesingle}. Note that we have fresh samples for every generator distribution, so Equation \eqref{eq:bound2} is true with high probability by union bound. For the real distribution, notice that Equation \eqref{eq:bound1} does not depend on the generator, so it is also true with high probability.
\end{proof}

\subsection{Omitted Proof for Section~\ref{sec:infinitemix}: Expressive power and existence of equilibrium} \label{sec:app:section4}

\paragraph{Mixed Equilibrium}

We first show there is a finite mixture of generators and discriminators that approximates the equilibrium of infinite mixtures.

Again we recall the settings here: suppose $\phi$ is $L_{\phi}$-Lipschitz and bounded in $[-\Delta,\Delta]$, the generator and discriminators are $L$-Lipschitz with respect to the parameters and $L'$-Lipschitz with respect to inputs.

\begin{theorem}[Theorem~\ref{thm:mixedequilibrium} restated]
In the settings above, if the generator can approximate any point mass\footnote{For all points $x$ and any $\epsilon>0$, there is a generator such that $\E_{h\sim\D_h}[\|G(h)-x\|] \le \epsilon$.}, there is a universal constant $C>0$ such that for any $\epsilon$, there exists $T = \frac{C\Delta^2 p\log (L L' L_{\phi}\cdot p/\epsilon)}{\epsilon^2}$ generators $G_{u_1},\ldots G_{u_T}$. % and $T$ discriminators $D_{v_1},\ldots D_{v_T}$, 
Let $\S_u$ be a uniform distribution on $u_i$, and $D$ is a discriminator that outputs only $1/2$, %and $\S_v$ be a uniform distribution on $v_i$, 
then $(\S_u,D)$ is an $\epsilon$-approximate equilibrium. %Furthermore, in this equilibrium the generator \textquotedblleft wins,\textquotedblright
%meaning discriminators cannot do better than random guessing.
\end{theorem}

%
%
%\begin{theorem}[Theorem~\ref{thm:mixedequilibrium} restated]
%In the settings described in the previous paragraph, there is a large enough constant $C>0$ such that for any $\epsilon$, there exists $T = \frac{C\Delta^2 p\log (L L' L_{\phi}\cdot p/\epsilon)}{\epsilon^2}$ generators $G_{u_1},\ldots G_{u_T}$ and $T$ discriminators $D_{v_1},\ldots D_{v_T}$, let $\S_u$ be a uniform distribution on $u_i$ and $\S_v$ be a uniform distribution on $v_i$, then $(\S_u,\S_v)$ is an $\epsilon$-approximate equilibrium.
%
%Further, if the class of generator can generate a Gaussian, and the class of discriminator includes constant functions, then the value of the game $V = 2\phi(1/2)$ (where $\phi$ is the connection function in \eqref{eq:gan-general}).
%\end{theorem}

\begin{proof}
We first prove the value $V$ of the game must be equal to $1/2$.For the discriminator, one strategy is to just output $1/2$. This strategy has payoff $2\phi(1/2)$ no matter what the generator does, so $V \ge 2\phi(1/2)$. 

For the generator, we use the assumption that for any point $x$ and any $\epsilon>0$, there is a generator (which we denote by $G_{x,\epsilon}$) such that $\E_{h\sim \D_h}[\|G_{x,\epsilon}(h) - x\|] \le \epsilon$.
Now for any $\zeta>0$, consider the following mixture  of generators: sample $x\sim D_{real}$, then use the generator $D_{x,\zeta}$. Let $\D_\zeta$ be the distribution generated by this mixture of generators. % $\D_\zeta = \D_{real}+\zeta N(0,I)$, which is the convolution of $\D_{real}$ and a Gaussian of variance $\zeta I$. For any $\zeta$, $\D_{\zeta}$ can be expressed as a infinite mixture of Gaussians and is therefore a mixed strategy of the generator. 
The Wasserstein distance between $\D_\zeta$ and $\D_{real}$ is bounded by $\zeta$. Since the discriminator is $L'$-Lipschitz, it cannot distinguish between $\D_\zeta$ and $\D_{real}$. In particular we know for any discriminator $D_v$
$$
|\E_{x\sim\D_\zeta}[\phi(1-D_v(x))] - \E_{x\sim\D_{real}}[\phi(1-D_v(x))]| \le O(L_{\phi}L'\zeta).
$$
Therefore,
\begin{align*}
& \max_{v\in \V} \E_{x\sim\D_{real}}[\phi(D_v(x))] + \E_{x\sim\D_\zeta}[\phi(1-D_v(x))] \\
\le & O(L_{\phi}L'\zeta) + \max_{v\in \V} \E_{x\sim\D_{real}}[\phi(D_v(x)) + \phi(1-D_v(x))] \\
 \le & 2\phi(1/2)+O(L_{\phi}L'\zeta).
\end{align*}
Here the last step uses the assumption that $\phi$ is concave. Therefore the value is upperbounded by $V \le 2\phi(1/2)+O(L_{\phi}L'\zeta)$ for any $\zeta$. Taking limit of $\zeta$ to $0$, we have $V = 2\phi(1/2)$.

The value of the game is $2\phi(1/2)$ in particular means the optimal discriminator cannot do anything other than a random guess. Therefore we will use a discriminator that outputs constant $1/2$. Next we will construct the generator.

Let ($\S'_u$,$\S'_v$) be the pair of optimal mixed strategies as in Theorem~\ref{thm:minmax} and $V$ be the optimal value. We will show that randomly sampling $T$ generators from $\S'_u$ gives the desired mixture with high probability.

Construct $\epsilon/4LL'L_{\phi}$-nets $V$ for the parameters of the discriminator $\mathcal{V}$. By standard construction, the sizes of these $\epsilon$-nets satisfy $\log (|V|) \le C'n\log (LL' L_{\phi}\cdot p/\epsilon)$ for some constant $C'$. Let $u_1,u_2,...,u_T$ be independent samples from $\S'_u$. By Chernoff bound, for any $v\in V$, we know
$$
\Pr[ \E_{i\in [T]}[F(u_i,v)] \ge \E_{u\in\mathcal{V}}[F(u,v)] + \epsilon/2] \le \exp(-\frac{\epsilon^2 T}{2\Delta^2}).
$$

When $T = \frac{C\Delta^2 p\log (L\cdot L_{\phi}\cdot p/\epsilon)}{\epsilon^2}$ and the constant $C$ is large enough ($C \ge 2C'$), with high probability this inequality is true for all $v\in V$. Now, for any $v\in \mathcal{V}$, let $v'$ be the closest point in the $\epsilon$-net. By the construction of the net, $\|v-v'\| \le \epsilon/4LL'L_{\phi}$. It is easy to check that $F(u,v)$ is $2LL'L_{\phi}$-Lipschitz in both $u$ and $v$, therefore
$$
\E_{i\in [T]}[F(u_i,v')] \le \E_{i\in [T]}[F(u_i,v)] + \epsilon/2.
$$
Combining the two inequalities we know for any $v'\in \mathcal{V}$, 
$$
\E_{i\in [T]}[F(u_i,v')] \le 2\phi(1/2)+\epsilon.
$$

This means the mixture of generators can win against any discriminator. By probabilistic argument we know there must exist such generators. The discriminator (constant $1/2$) obviously achieve value $V$ no matter what the generator is. Therefore we get an approximate equilibrium.
%This finishes the proof for the first inequality.  

%Finally, we prove the fact that the value $V$ must be equal to $2\phi(1/2)$. 
\end{proof}

\paragraph{Pure equilibrium} %Now we show for Wasserstein objective, there exists an approximate pure equilibrium

Now we show how to construct a larger generator network that gives a pure equilibrium.

\begin{theorem}[Theorem~\ref{thm:pureequilibrium} restated]
Suppose the generator and discriminator are both $k$-layer neural networks ($k\ge 2$) with $p$ parameters, and the last layer uses ReLU activation function. In the setting of Theorem~\ref{thm:mixedequilibrium}, there exists $k+1$-layer neural networks of generators $G$ and discriminator $D$ with $O\left(\frac{\Delta^2 p^2\log (L L' L_{\phi}\cdot p/\epsilon)}{\epsilon^2}\right)$ parameters, such that there exists an $\epsilon$-approximate pure equilibrium with value $2\phi(1/2)$.
\end{theorem}

%
%\begin{theorem}[Theorem~\ref{thm:pureequilibrium} restated]
%Suppose the generator and discriminator are both $k$-layer neural networks ($k\ge 2$) with $p$ parameters, and the last layer uses ReLU activation function. In the setting of Theorem~\ref{thm:mixedequilibrium} there exists $k+1$-layer neural networks of generators $G$ and discriminator $D$ with $O\left(\frac{\Delta^2 p^2\log (L L' L_{\phi}\cdot p/\epsilon)}{\epsilon^2}\right)$ parameters, such that there exists an $\epsilon$-approximate pure equilibrium. Furthermore, if the generator is capable of generating a Gaussian then the value $V = 1$.
%\end{theorem}

In order to prove this theorem, the major step is to construct a generator that works as a mixture of generators. More concretely, we need to construct a single neural network that approximately generates the mixture distribution using the gaussian input it has.
%\paragraph{Mixture of Generators} For mixture of generators,  
To do that, we can pass the input $h$ through all the generators $G_{u_1}, G_{u_2},..., G_{u_T}$. We then show how to implement a ``multi-way selector'' that will select a uniformly random output from $G_{u_i}(h)$ ($i\in [T]$).

We first observe that it is possible to compute a step function using a two layer neural network. This is fairly standard for many activation functions.

\begin{lemma} \label{lem:stepfunc}
Fix an arbitrary $q\in \N$ and $z_1 < z_2 < \cdots < z_q$. For any $0 <\delta <\min\{z_{i+1}-z_i\}$, there is a two-layer neural network with a single input $h\in \mathbb{R}$ that outputs $q+1$ numbers $x_1,x_2,...,x_{q+1}$ such that (i) $\sum_{i=1}^{q+1} x_i = 1$ for all $h$; (ii) when $h\in [z_{i-1}+\delta/2,z_i-\delta/2]$, $x_i = 1$ and all other $x_j$'s are 0\footnote{When $h \le z_1 -\delta/2$ only $x_1$ is 1 and when $h \ge z_q+\delta/2$ only $x_{q+1} = 1$}.
\end{lemma}

\begin{proof}
Using a two layer neural network, we can compute the function $f_i(h) = \max\{\frac{h-z_i-\delta/2}{\delta},0\} - \max\{\frac{h-z_i+\delta/2}{\delta},0\}$. This function is 0 for all $h < z_i-\delta/2$, 1 for all $h \ge z_i+\delta/2$ and change linearly in between. Now we can write $x_1 = 1 - f_1(h)$, $x_{q+1} = f_q(h)$, and for all $i=2,3,...,q$, $x_q = f_i(h) - f_{i-1}(h)$. It is not hard to see that these functions satisfy our requirements.
\end{proof}

Using these step functions, we can essentially select one output from the $T$ generators.

\begin{lemma}\label{lem:select}
In the setting of Theorem~\ref{thm:pureequilibrium}, for any $\delta > 0$, there is a $k+1$-layer neural network with $O\left(\frac{\Delta^2 p^2\log (L L' L_{\phi}\cdot p/\epsilon)}{\epsilon^2}\right)$ parameters that can generate a distribution that is within $\delta$ total variational difference with the mixture of $G_{u_1}, G_{u_2}, ..., G_{u_T}$.
\end{lemma}

The idea is simple: since we have implemented step functions from Lemma~\ref{lem:stepfunc}, we can just pass through the input through all the generators $G_{u_1}, ..., G_{u_T}$. For the last layer of $G_{u_i}$, we add a large multiple of $-(1-x_i)$ where $x_i$ is the $i$-th output from the network in Lemma~\ref{lem:stepfunc}. Clearly, if $x_i = 0$ this is going to effectively disable the neural network; if $x_i = 1$ this will have no effect. By properties of $x_i$'s we know most of the time only one $x_i = 1$, hence only one generator is selected.

\begin{proof}
Suppose the input for the generator is $(h_0,h)\sim N(0,1)\times \D_h$ (i.e. $h_0$ is sampled from a Gaussian, $h$ is sampled according to $\D_h$ independently). We pass the input $h$ through the generators and gets outputs $G_{u_i}(h)$, then we use $h_0$ to select one as the true output. 

Let $z_1,z_2,...,z_{T-1}$ be real numbers that divides the probability density of a Gaussian into $T$ equal parts. Pick $\delta' = \delta/100T$ in Lemma~\ref{lem:stepfunc}, we know there is a 2-layer neural network that computes step functions $x_1,...,x_{T}$. Moreover, the probability that $(x_1,...,x_{T})$ has more than 1 nonzero entry is smaller than $\delta$. Now, for the output of $G_{u_i}(h)$, in each output ReLU gate, we add a very large multiple of $-(1-x_i)$ (larger than the maximum possible output). This essentially ``disables'' the output when $x_i = 0$ because before the result before ReLU is always negative. On the other hand, when $x_i = 1$ this preserves the output. Call the modified network $\hat{G}_{u_i}$, we know $\hat{G}_{u_i} = G_{u_i}$ when $x_i = 1$ and $\hat{G}_{u_i} = 0$ when $x_i = 0$.  Finally we add a layer that outputs the sum of $\hat{G}_{u_i}$. By construction we know when $(x_1,...,x_T)$ has only one nonzero entry, the network correctly outputs the corresponding $G_{u_i}(x_i)$. The probability that this happens is at least $1-\delta$ so the total variational distance with the mixture is bounded by $\delta$.
\end{proof}

Using the generator  we constructed, it is not hard to prove Theorem~\ref{thm:pureequilibrium}. The only thing to notice here is that when the generator is within $\delta$ total variational distance to the true mixture, the payoff $F(u,v)$ can change by at most $2\Delta \delta$. 

\begin{proof}[Proof of Theorem~\ref{thm:pureequilibrium}] Let $T$ be large enough so that there exists an $\epsilon/2$-approximate mixed equilibrium. %Let the new set of discriminators be the convex combination of $T$ discriminators $\{D_v,v\in \V\}$. 
Let the new set of generators be constructed as in Lemma~\ref{lem:select} with $\delta\le \epsilon/4\Delta$ and $G_{u_1},...,G_{u_T}$ from the original set of generators. Let $D$ be the discriminator that always outputs $1/2$,% which is the average of the $T$ discriminators from the approximate mixed equilibrium,
 and $G$ be the generator constructed by the $T$ generators from the approximate mixed equilibrium. Define $F^\star(G,D)$ be the payoff of the new two-player game. Now, for any discriminator $D_v$, we have%think of it as a distribution of $G_v$, we know
\begin{align*}
F^\star(G,v) &\le \E_{i\in[T]}F(u_i,v) + |F^\star(G,D')-\E_{i\in[T]}F(u_i,v)|\\
& \le V+\epsilon/2 + 2\Delta\frac{\epsilon}{4\Delta} \\
& \le V+\epsilon.
\end{align*}
The bound from the first term comes from Theorem~\ref{thm:mixedequilibrium}, and the fact that the expectation is smaller than the max. The bound for the second term comes from the fact that changing a $\delta$ fraction of probability mass can change the payoff $F$ by at most $2\Delta \delta$. Therefore the generator can still fool all discriminators, and we get a pure equilibrium.

%Similarly, for any generator $G'$, we know it is close to a mixture of generators $G_u$, therefore
%\begin{align*}
%F^\star(G',D) &\le \E_{i\in[T],u\in G'}F(u,v_i)\\
%& + |F^\star(G',D)-\E_{i\in[T],u\in G'}F(u,v_i)|\\
%& \le  V+\epsilon/2 + 2\Delta\frac{\epsilon}{4\Delta} \\
%& \le V+\epsilon.
%\end{align*}
%This finishes the proof.
\end{proof}

\section{Examples when best response fail}
\label{sec:examples}
In this section we construct simple generators and discriminators and show that if both generators and discriminators are trained to optimal with respect to Equation~\eqref{eq:gan-general}, the solution will cycle and cannot converge. For simplicity, we will show this when the connection function is $\phi$, but it is possible to show similar result even when $\phi$ is the traditional $\log$ function.

We consider a simple example where we try to generate points on a circle. The true distribution $\mathcal{D}_{true}$ has $1/3$ probability to generate a point at angle $0, 2\pi/3, 4\pi/3$. Let us first consider a case when the generator does not have enough capacity to generate the true distribution. 

\begin{definition}[Example 1] The generator $G$ has one parameter $\theta \in [0,2\pi)$, and always generates a point at angle $\theta$. The discriminator $D$ has a parameter $\phi \in [0,2\pi)$, and $D_\phi(\tau) = \exp(-10d(\tau,\phi)^2)$. Here $d(\tau,\phi)$ is the angle between $\tau$ and $\phi$ and is always between $[0,\pi]$.
\end{definition}

We will analyze the ``best response'' procedure (as in Algorithm~\ref{alg:bestresponse}). We say the procedure converges if $\lim_{i\to\infty} \phi^i$ and $\lim_{i\to\infty} \theta^i$ exist.

\begin{algorithm}
\begin{algorithmic}
\STATE Initialize $\theta^0$.
\FOR{$i = 1$ to $T$}
\STATE Let $\phi^i = \arg\max_\phi \E_{\tau\sim \mathcal{D}_{true}}[D_\phi(\tau)] - \E_{\tau\sim  G(\theta)}[D_\phi(\tau)].$
\STATE Let $\theta^i = \arg\min_\theta - \E_{\tau\sim  G(\theta)}[D_\phi(\tau)].$
\ENDFOR
\end{algorithmic}
\caption{Best Response}\label{alg:bestresponse}
\end{algorithm}

\begin{theorem}
For generator $G$ and discriminator $D$ in Example 1, for every choice of parameter $\theta$, there is a choice of $\phi$ such that $D_\phi(\theta) \le 0.001$ and $\E_{\tau\sim \mathcal{D_{true}}}[D_\phi(\tau)] \ge 1/3$. On the other hand, for every choice of $\phi$, there is always a $\theta$ such that $D_\phi(\theta) = 1$. As a result, the sequence of best response cannot converge.
\end{theorem}

\begin{proof}
For any $\theta$, we can just choose $\phi$ to be the farthest point in $\{0,2\pi/3,4\pi/3\}$. Clearly the distance is at least $2\pi/3$ and therefore $D_\phi(\theta) \le \exp(-10) \le 0.001$. On the other hand, for the true distribution, it has $1/3$ probability of generating the point $\phi$, therefore $\E_{\tau\sim \mathcal{D_{true}}}[D_\phi(\tau)] \ge 1/3$. For every $\phi$, we can always choose $\theta = \phi$, and we have $D_\phi(\theta) = 1$.

By the construction of $\phi^i$ and $\theta^i$ in Algorithm~\ref{alg:bestresponse}, we know for any $i$ $D_{\phi^i}(\theta^i) = 1$, but $\E_{\tau\sim \mathcal{D}_{true}}[D_\phi(\tau)] - D_{\phi^{i-1}}(\theta^i)\ge 1/4$. Therefore $|D_{\phi^i}(\theta^i) - D_{\phi^{i-1}}(\theta^i)| \ge 1/4$ for all $i$ and the sequences cannot converge.
\end{proof}

This may not be very surprising as the generator does not even have the capacity to generate the true distribution. One might hope that once we use a large enough neural network, the generator will be able to generate the true distribution. However, our next example shows even in that case the best response algorithm may not converge. 

\begin{definition}[Example 2]Let $\theta, \phi \in [0,2\pi)^3$ will be 3 -dimensional vectors. The generator $G(\theta)$ generates the uniform distribution over points $\theta_1,\theta_2,\theta_3$. The discriminator function is chosen to be $D_\phi(\tau) = \frac{1}{3}\sum_{i=1}^3\exp(-10d(\tau,\phi_i)^2)$.
\end{definition}

Clearly in this example, the true distribution can be generated by the generator (just by choosing $\theta = (0,2\pi/3,4\pi/3)$). However we will show that the best response algorithm still cannot always converge.

\begin{theorem} Suppose the generator and discriminator are described as in Example 2, and $\theta^0 = (0,0,0)$, then we have: (1) In every iteration the three points for generator $\theta^i_{1,2,3}$ are equal to each other. (2) In every iteration $\theta^i_1$ is $0.1$-close to one of the true points $\{0,2\pi/3,4\pi/3\}$, and its closest point is different from the previous iteration.
\end{theorem}

Before giving detailed calculations, we first give an intuitive argument. In this example, we will use induction to prove that at every iteration $t$, two properties are preserved: 1. The three points of the generator ($\theta^t_1,\theta^t_2,\theta^t_3$) are close to the same real example ($0,2\pi/3$ or $4\pi/3$); 2. The three points of the discriminator ($\phi^{t+1}_1,\phi^{t+1}_2,\phi^{t+1}_3$) will be close to the other two real examples. To go from 1 to 2, notice that in this case the three $\phi$ values can be optimized independently (and the final objective is the sum of the three), so it suffices to argue for one of them. For one $\phi$, by our construction the objective function is really close to the sum of two Gaussians at the other two real examples, minus twice of a Gaussian at the real example that $\theta^t_i$'s are close to (see Figure~\ref{fig:opt}). From the Figure it is clear that the maximum of this function is close to one of the real examples that is different from $\theta^t_i$.Now all the three $\phi^{t+1}_i$'s will be close to one of the two real examples, so one of them is going to get at least two $\phi^{t+1}_i$'s. In the next iteration, as the generator is trying to maximize the output of discriminator, all three $\theta^{t+1}_i$'s will be close to the real example with at least two $\phi^{t+1}_i$'s.

\begin{figure}
\centering
\includegraphics[width=3in]{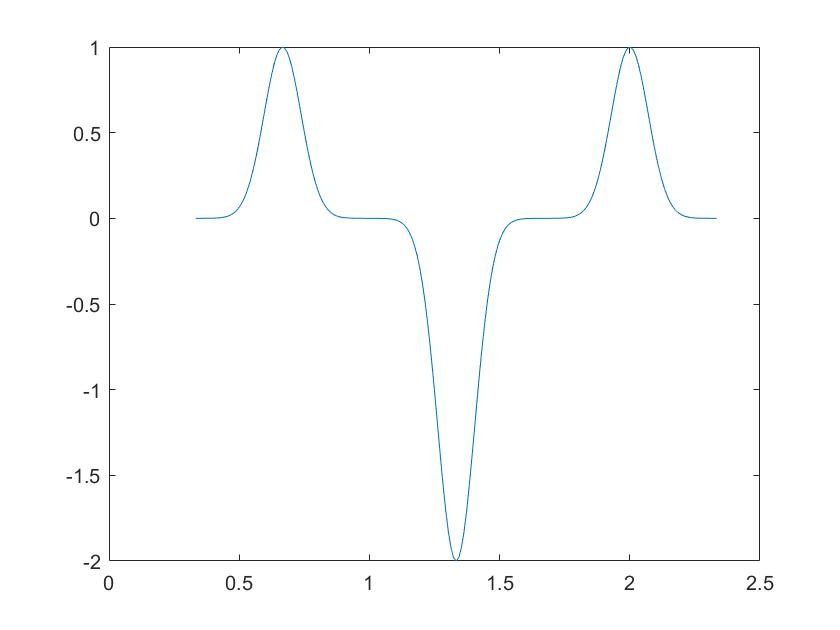}
\caption{Optimization problem for $\phi$\label{fig:opt}}
\end{figure}

Now we make this intuition formal through calculations.

\begin{proof}
The optimization problems here are fairly simple and we can argue about the solutions directly. Throughout the proof we will use the fact that $\exp(-10*(2\pi/3)^2) < 1e-4$ and $\exp(1+\epsilon) \approx 1+\epsilon$. The following claim describes the properties we need:

\begin{claim}
If $\theta_1=\theta_2=\theta_3$ and $\theta_1$ is $0.1$-close to one of $\{0,2\pi/3,4\pi/3\}$, then the optimal $\phi$ must satisfy $\phi_1,\phi_2,\phi_3$ be $0.05$-close to the other two points in $\{0,2\pi/3,4\pi/3\}$.
\end{claim}

We first prove the Theorem with the claim. We will do this by induction. The induction hypothesis is that for every $j \le t$, we have $\theta^j_{1,2,3}$ are equal to each other, and $\theta^j_1$ is $0.1$-close to one of the true points $\{0,2\pi/3,4\pi/3\}$. This is clearly true for $t = 0$. Now let us assume this is true for $t$ and consider iteration $t+1$.

Without loss of generality we assume $\theta^t_1$ is close to $0$. Now by Claim we know $\phi^t_1,\phi^t_2,\phi^t_3$ are 0.05-close to either $2\pi/3$ or $4\pi/3$. Without loss of generality assume there are more $\phi^t_i$'s close to $2\pi/3$ (the number of $\phi^t_i$'s close to the two point has to be different because there are 3 $\phi^t_i$'s). Now, by property of Gaussians, we know $D_{\phi^t}(\tau)$ has a unique maximum that is within $0.1$ of $2\pi/3$ (the influence from the other point is much smaller than $0.05$). Since the generator is trying to maximize $\E_{\tau\sim  G(\theta)}[D_\phi(\tau)]$, all three $\theta^t_i (i=1,2,3)$ should be at this maximizer. This finishes the induction. 
\end{proof}

Now we prove the claim.

\begin{proof}
The objective function is
$$
\max_\phi \frac{1}{3}\left(\E_{\tau\sim \mathcal{D}_{true}}[\sum_{i=1}^3\exp(-10d(\tau,\phi_i)^2)] - \E_{\tau\sim  G(\theta)}[\sum_{i=1}^3\exp(-10d(\tau,\phi_i)^2)]\right).
$$ 
In this objective function $\phi_i$'s do not interact with each other, so we can break the objective function into three (one for each $\phi_i$). Without loss of generality, assume $\theta_{1,2,3}$ are $0.1$-close to 0, we are trying to maximize 
$$
\max_\phi \left(\frac{1}{3}\sum_{i=1}^3\exp(-10d(\phi,i \cdot 2\pi/3)^2)] - \exp(-10d(\theta,\phi)^2)\right).
$$ 

Clearly, if $\phi$ is not $0.05$ close to either $2\pi/3$ or $4\pi/3$, we have $D \le 1/3 - 10\cdot 0.05^2 + \exp(-10) \le 1/3 - 0.02$. On the other hand, when $\phi = 2\pi/3$ or $4\pi/3$, we have $D \ge 1/3 - \exp(-10)$. Therefore the maximum must be $0.05$-close to one of the two points.
\end{proof}

Note that this example is very similar to the {\em mode collapse} problem mentioned in NIPS 2016 GAN tutorial\citep{goodfellow2016nips}.

\end{document}